\documentclass[letterpaper,12pt]{article}

\usepackage{times}
\usepackage[margin=1in]{geometry}

\usepackage{graphicx}
\DeclareGraphicsExtensions{.pdf,.jpeg,.png}
\usepackage{cite}
\usepackage{amsmath}
\usepackage{amsthm}
\usepackage{amssymb}
\usepackage{amsfonts}
\usepackage{titlesec}
\usepackage{enumerate}
\usepackage[linesnumbered,ruled]{algorithm2e}
\usepackage[justification=centering]{caption}
\titlelabel{\thetitle.\quad}
\titleformat*{\section}{\normalfont\bfseries}
\titleformat*{\subsection}{\normalfont\bfseries}
\titleformat*{\subsubsection}{\normalfont\bfseries}
\titleformat*{\paragraph}{\normalfont\bfseries}
\titleformat*{\subparagraph}{\normalfont\bfseries}

\newtheorem{remark}{Remark}
\newtheorem{lemma}{Lemma}
\newtheorem{definition}{Definition}

\usepackage{hyperref}
\hypersetup{hypertex=true,
	colorlinks=true,
	linkcolor=red,
	anchorcolor=blue,
	citecolor=green}
\usepackage[all]{xy}
\usepackage{arydshln}

\begin{document}
\date{}

\title{Equivariant Filtering Framework for Inertial-Integrated Navigation}

\author{Yarong Luo, 
	yarongluo@whu.edu.cn\\
 	Chi Guo,
	guochi@whu.edu.cn\\
	Jingnan Liu,
	jnliu@whu.edu.cn\\
 	GNSS Research Center, Wuhan University
}



\maketitle

\thispagestyle{empty}

\noindent
{\bf\normalsize Abstract}\newline
{This paper proposes a equivariant filtering (EqF) framework for the inertial-integrated state estimation problem. As the kinematic system of the inertial-integrated navigation can be naturally modeling on the matrix Lie group $SE_2(3)$, the symmetry of the Lie group can be exploited to design a equivariant filter which extends the invariant extended Kalman filtering on the group affine system. Furthermore, details of the analytic state transition matrices for left invariant error and right invariant error are given.
} \vspace{2ex}

\noindent
{\bf\normalsize Key Words}\newline
{Equivariant Filtering, Left Equivariant System, $SE_2(3)$ matrix Lie group, inertial-integrated navigation, Group-affine Dynamics, Full Second Order Kinematic Systems, Analytic State Transition Matrix}


\section{Introduction}
{A}{s} dynamic systems on Lie groups are common in the robotics and avionics,
the development of robust and accurate state estimation algorithms for autonomous navigation system arouses a great deal of interest in robotics and avionics communities, especially when the states of the vehicles are evolving on Lie group. 
The symmetries of the system models are exploited to design observers and filters for attitude~\cite{ng2019attitude} and pose estimation~\cite{hua2015gradient}, tracking of homographies~\cite{hua2020nonlinear}, velocity aided attitude estimation~\cite{bonnable2009invariant}, and visual simultaneous localization and mapping (VSLAM)~\cite{zlotnik2018gradient,mahony2020equivariant}.
It is shown that any kinematic system on a Lie group can be embedded in a natural manner into an equivalent kinematic system and the equivariance properties of the system embedding can be applied to design equivalent filter for any kinematic system on a Lie group.

The $SE_2(3)$ based EKF framework proposed recently exploited the explicit algebraic invariance condition termed group affine that characterises the inertial-integrated navigation systems~\cite{luo2021se23}.
Robert et.al~\cite{mahony2020equivariant4} have shown that the equivalent embedding of group affine systems only requires a finite dimensional input vector space extension.
Furthermore, the research about equivariant observer design for any kinematic system on a homogeneous space has been studied in~\cite{mahony2020equivariant6} .
Therefore, it is attractive to embed the inertial-integrated navigation system into an equivalent system so that an equivalent filter can be designed as a kinematic system on a Lie group.
Equivariant system theory has been used for full second order kinematic systems on $\mathrm{T}\mathbf{SO}(3)$~\cite{ng2019attitude} and $\mathrm{T}\mathbf{SE}(3)$~\cite{ng2020equivariant}.
Motivated by the equivalent observer design for equivalent kinematic system, we propose an equivalent filter framework for the inertial-integrated navigation by leveraging the Lie group structure of $SE_2(3)$, which can be viewed as equivariant filter design for second order kinematic systems on $\mathrm{T}\mathbf{SE}_2(3)$.

The contributions of the paper can be summarized as follows:

1. We propose an equivalent filter framework for inertial-integrated navigation system which embedded the attitude, velocity, and position into the matrix Lie group $SE_2(3)$.

2. We consider the left equivalent system and its associated properties as well as relationships with the group affine systems when confronting the full second order kinematic systems on $\mathrm{T}\mathbf{SE}_2(3)$.

3. We give detailed derivations of the equivalent filter for inertial-integrated navigation on the navigation frame and the earth frame.

4. We analytically derive the block entries of the state transition matrices in details.

This remainder of this paper is organized as follows. Preliminaries are presented in Section II. 
In section III theory for left equivariant system is introduced. In section IV Full Second Order Kinematic System in navigation frame and earth frame are given. The equivariant filtering for second order kinematics systems on $TSE_2(3)$ is shown in Section V. Then, equivariant filtering design for inertial-integrated navigation system is shown in Section VI. Conclusion and future work are given in Section VII.
\section{Preliminaries}
The kinematics of the vehicles are described by the velocity, position and the direction, which are expressed on the manifold space and identified by different frames. The velocity and position can be represented by the vectors and the attitude in the 3-dimensional vector space can be represented by the direction cosine matrix (DCM). 
This three quantities can be reformulated as an element of the $SE_2(3)$ matrix Lie group. The $SE_2(3)$ matrix Lie group is also called the group of direct spatial isometries~\cite{barrau2017the} and it represents the space of matrices that apply a rigid body rotation and 2 translations to points in $\mathbb{R}^3$. Moreover, the group $SE_2(3)$ has the structure of the semidirect product of SO(3) group by $\mathbb{R}^3\times\mathbb{R}^3$ and can be expressed as $SE_2(3)=SO(3)\ltimes \underbrace{\mathbb{R}^3\times\mathbb{R}^3}_2$~\cite{luo2020geometry}. The relationship between the Lie algebra and the associated vector is described by a linear isomorphism $\Lambda$: $\mathbb{R}^9 \rightarrow \mathfrak{se}_2(3)$. More details about $SE_2(3)$ matrix Lie group can be found in~\cite{barrau2017the,luo2020geometry}. 
The uncertainties on matrix Lie group $SE_2(3)$ can be represented by left multiplication and right multiplication and lead to left-invariant error and right-invariant. The invariant property can be verified by left group action and right group action.

Meanwhile, the vector $v_{ab}^c$ describes the vector points from point a to point b and expressed in the c frame. The direction cosine matrix $C_d^f$ represents the rotation from the d frame to the f frame. 
The commonly used reference frames in inertial-integrated navigation systems are summarized as earth-centered-inertial (ECI) frames (i-frame), earth-centered-earth-fixed (ECEF) frames (e-frame), north-east-down navigation frames (n-frame), and forward-transversal-down body frames (b-frame)~\cite{shin2005estimation}.

The "plumb-bob gravity"~\cite{savage2000strapdown} is given as
\begin{equation}\label{plumb_bob_gravity}
g_{ib}=G_{ib}-(\omega_{ie}\times)^2r_{eb}
\end{equation}
where $g_{ib}$ is the gravity vector; $G_{ib}$ is the gravitational vector.
This formula can be expressed in ECI frame, ECEF frame, and navigation frame.
The perturbation on the gravity expressed in ECEF frame can be written as~\cite{groves2013principles}
\begin{equation}\label{perturbation_g_ib_e}
\delta g_{ib}^e\triangleq \tilde{g}_{ib}^e-g_{ib}^e\approx -\frac{\mu}{||r_{ib}^e||^3}\delta r_{ib}^e
\end{equation}
where $\mu$ is defined in Chapter2 of~\cite{groves2013principles}.

The perturbation on the gravity expressed in navigation frame can be written as
~\cite{shin2005estimation}
\begin{equation}\label{perturbation_gravity}
\delta g_{ib}^n\triangleq\tilde{g}_{ib}^n-g_{ib}^n \approx\begin{bmatrix}
0&0& \frac{2g_{ib}^n}{\sqrt{R_MR_N}+h}\delta r_D
\end{bmatrix}^T
\end{equation}
where $\sqrt{R_MR_N}$ is the Gaussian mean Earth radius of curvature; $\delta r_D$ is perturbation of the error position vector $\delta r_{eb}^n$ in the down direction of NED frame.

If the biases, scale factors, and non-orthogonalities of the accelerometers and gyroscopes are considered, then the uncertainty of the sensors can be expressed as~\cite{shin2005estimation}
\begin{equation}\label{uncertainty_sensors}
\begin{aligned}
\delta f_{ib}^b&=b_a+diag(f_{ib}^b)s_a+\Gamma_a\gamma_a\\
\delta \omega_{ib}^b&=b_g+diag(\omega_{ib}^b)s_g+\Gamma_g\gamma_g
\end{aligned}
\end{equation} 
where $b_a$ and $b_g$ are residual biases of the accelerometers and gyroscopes, respectively; $s_a$ and $s_g$ are the scale factors of the accelerometers and gyroscopes, respectively; $\gamma_a$ and $\gamma_g$ are the non-orthogonalities of the accelerometer triad and gyroscope triad, respectively.
$diag(a)$ represents the diagonal matrix form of a 3-dimensional vector $a$. $\Gamma_a$ and $\Gamma_g$ can be found in~\cite{shin2005estimation}.
The random constant, the random walk and the first-order Gauss-Markov models are typically used in modeling the inertial sensor errors~\cite{shin2005estimation}.

Finally, we summarize the commonly used frames in the inertial navigation and give detailed navigation equations in both the NED frame and the ECEF frame.
The position error state expressed in radians is usually very small, which will cause numerical instability in Kalman filtering calculation. Therefore, it is usually to represent the position error vector in terms of the XYZ coordinate system.
The position vector differential equation in terms of the NED frame can be calculated as
\begin{equation}\label{NED_position_1}
\dot{r}_{eb}^n=-\omega_{en}^n\times r_{eb}^n+v_{eb}^n
\end{equation}

The differential equation of the velocity vector in the NED frame is given by
\begin{equation}\label{NED_velocity_1}
\dot{v}_{eb}^n=C_b^nf_{ib}^b-[(2\omega_{ie}^n+\omega_{en}^n)\times]v_{eb}^n+g_{ib}^n
\end{equation}
where $\omega_{ie}^n$ is the earth rotation vector expressed in the navigation frame; $f_{ib}^b$ is the specific force vector in navigation frame; $\omega_{en}^n=\omega_{in}^n-\omega_{ie}^n$ is the angular rate vector of the navigation frame relative to the ECEF frame expressed in the navigation frame which is also called the transport rate.

The attitude in the NED frame can be represented by the direction cosine matrix $C_b^n$ and its differential equation is given by
\begin{equation}\label{NED_attitude_1}
\dot{C}_{b}^n=C_b^n(\omega_{ib}^b\times)-(\omega_{in}^n\times)C_b^n
\end{equation}

When the attitude, velocity, and position on NED frame are represented as $C_b^n$, $v_{ib}^n$, and $r_{ib}^n$, their differential equations are also considered. They are given as
\begin{equation}\label{ib_n}
\begin{aligned}
\dot{C}_{b}^n&=C_b^n(\omega_{ib}^b\times)-(\omega_{in}^n\times)C_b^n\\
\dot{v}_{ib}^n&=-\omega_{in}^n\times v_{ib}^n+C_b^nf_{ib}^b+G_{ib}^n\\
\dot{r}_{ib}^n&=-\omega_{in}^n\times r_{ib}^n+v_{ib}^n
\end{aligned}
\end{equation}

When the attitude, velocity, and position on ECEF frame are represented as $C_b^e$, $v_{eb}^e$, and $r_{eb}^e$, their differential equations are also considered. They are given as
\begin{equation}\label{eb_e}
\begin{aligned}
\dot{C}_{b}^e&=C_b^e(\omega_{ib}^b\times)-(\omega_{ie}^e\times)C_b^e\\
\dot{v}_{eb}^e&=-2\omega_{ie}^e\times v_{eb}^e+C_b^ef_{ib}^b+g_{ib}^e\\
\dot{r}_{eb}^e&=v_{eb}^e
\end{aligned}
\end{equation}

When the attitude, velocity, and position on ECEF frame are represented as $C_b^e$, $v_{ib}^e$, and $r_{ib}^e$, their differential equations are also considered. They are given as
\begin{equation}\label{ib_e}
\begin{aligned}
\dot{C}_{b}^e&=C_b^e(\omega_{ib}^b\times)-(\omega_{ie}^e\times)C_b^e\\
\dot{v}_{ib}^e&=-\omega_{ie}^e\times v_{ib}^e+C_b^ef_{ib}^b+G_{ib}^e\\
\dot{r}_{ib}^e&=-\omega_{ie}^e\times r_{ib}^e+v_{ib}^e
\end{aligned}
\end{equation}
\subsection{Lie group and Lie algebra}
A left group action $\psi$ of G on a smooth manifold $\mathcal{M}$ is a smooth mapping 
\begin{equation}\label{left_action}
\psi:G\times \mathcal{M}\rightarrow \mathcal{M}
\end{equation}
with $\psi(A,\psi(B,x))=\psi(AB,x)$ and $\psi(I,v)=v$ for all $A,B\in G$ and $v\in V$. 
A group action induces families of smooth diffeomorphisms $\psi_A:\mathcal{M}\rightarrow \mathcal{M}$ for $A\in G$ by $\psi_{A}(x):=\psi(A,x)$, and smooth nonlinear projection $\psi_{x}:G\rightarrow \mathcal{M}$ for $x\in \mathcal{M}$ by $\psi_{x}(A):=\psi(A,x)$. 
The left action $\psi$ is called linear if it induces linear mapping $\psi_{A}$.
\begin{lemma}
	Define a smooth map $d_{\star}L:\mathcal{M}\times\mathfrak{X}(\mathcal{M})\rightarrow\mathfrak{X}(\mathcal{M})$, by 
	\begin{equation}\label{linear_left_action_deduce}
	d_{\star}L(Z,F):=dL_{Z}\cdot F\circ L_{Z^{-1}}\in \mathfrak{X}(\mathcal{M})
	\end{equation}
	Then $d_{\star}L$ is a linear left group action on the vector space $\mathfrak{X}(\mathcal{M})$.
\end{lemma}
\begin{proof}
	Note that for $A\in \mathcal{M}$, $L_{A}$ is a diffeomorphism with smooth inverse $L_{A^{-1}}$. 
	We have\begin{equation}\label{left_action_deduce_proof}
	\begin{aligned}
	d_{\star}L(A,d_{\star}L(B,F))&=dL_{A}\cdot (dL_{B}\cdot F\circ L_{B^{-1}})\circ L_{A^{-1}})\\
	&=dL_{AB}\cdot F\circ L_{(AB)^{-1}}\\
	&=d_{\star}L(AB,F)
	\end{aligned}
	\end{equation}
	for all $A.B\in G$ and $F\in\mathfrak{X}(\mathcal{M})$ since L is a left action on $\mathcal{M}$. The identity property of the group action is straightforward and this demonstrates $d_{\star}L$ is a group action. Linearity follows since
	\begin{equation}\label{linearity_proof}
	\begin{aligned}
	d_{\star}L(A,\alpha_1F_1+\alpha_2F_2)&=dL_{A}\cdot (A,\alpha_1F_1+\alpha_2F_2)\circ L_{A^{-1}} \\
	&=\alpha_1dL_A\cdot F_1\circ L_{A^{-1}}+\alpha_2ddL_A\cdot F_2\circ L_{A^{-1}}\\
	&=\alpha_1d_{\star}L(A,F_1)+\alpha_2d_{\star}L(A,F_2)
	\end{aligned}
	\end{equation}
	for all $A\in \mathcal{M}$, $F_1,F_2\in\mathfrak{X}(\mathcal{M})$ and $\alpha_1,\alpha_2\in \mathbb{R}$.
\end{proof}
\begin{remark}
	The linear group action $d_{\star}L$ defines a representation of the manifold $\mathcal{M}$ on the infinite dimensional vector space $\mathfrak{X}(\mathcal{M})$. We will see that equivariant kinematic systems correspond precisely to subrepresentations of this representation~\cite{mahony2020equivariant4}.
\end{remark}

We introduce an algebraic object that is closely related to the system function $\mathcal{F}$ and is important in understanding invariant system structures.
\begin{definition}\label{lift_definition}
	Let $\mathcal{F}:V\rightarrow \mathfrak{X}(\mathcal{M})$ be a kinematic system on a manifold $\mathcal{M}$ over a vector space $V$. The lift is the function $\Lambda: \mathcal{M}\times V\rightarrow \mathfrak{g}$ defined by
	\begin{equation}\label{lift}
	\Lambda(X,v):=\mathcal{F}_{v}(X)X^{-1}=dR_{X^{-1}}\mathcal{F}_{v}(X)
	\end{equation}
	for all $X\in \mathcal{M}$ and $v\in V$.
\end{definition}

The lift provides an algebraic structure that connects the input vector space $V$ to the Lie algebra $\mathfrak{g}$.
Properties of equivariance, invariance and being group affine can be expressed as algebraic properties of the map $\Lambda$ that hold on vector subspaces of $V$.

\section{Left Equivariant Systems}
A left equivariant system is one in which left translation of the system function is the same as evaluating the function at the translated base point along with a possible group action transformation of the input space~\cite{mahony2020equivariant6}.
\begin{definition}\label{definition_equi_sys}
	(Left Equivariant System) A kinematic system $\mathcal{F}: V\rightarrow \mathfrak{X}(\mathcal{M})$ is left equivariant if there exists a left group action $\psi: G\times V\rightarrow V$ such that 
	\begin{equation}\label{left_equivariant_definition}
	dL_A\mathcal{F}_v(X)=\mathcal{F}_{\psi_{A}(v)}(L_A(X))
	\end{equation}
	for all $A$, $X\in G$ and $v\in V$. 
\end{definition}
Assume that a system $\mathcal{F}:V\rightarrow \mathfrak{X}(\mathcal{M})$ is left equivariant according to Definition \ref{definition_equi_sys}. Then 
\begin{equation}\label{equivariant_intepreter}
\begin{aligned}
\mathcal{F}_{\psi_{A}(v)}(X)&=\mathcal{F}_{\psi_{A}(v)}(L_A(A^{-1}X))\\
&=dL_A\mathcal{F}_v(A^{-1}X)\\
&=dL_A\cdot \mathcal{F}_v \circ L_{A^{-1}}(X)\\
&=d_{\star}L_{A}\mathcal{F}_{v}(X)
\end{aligned}
\end{equation}
where the second line follows from equation (\ref{left_equivariant_definition}). In particular, the input group action $\psi$ is uniquely determined by the group action $d_{\star}L$ on the vector field.

Another way of reading equation (\ref{equivariant_intepreter}) is that for an equivariant kinematic system, the vector subspace $im\mathcal{F}\subset \mathfrak{X}(\mathcal{M})$ is invariant under the group action $d_{\star}L$. 

Conversely, given such an invariant subspace $V\subset \mathfrak{X}(\mathcal{M})$, define a kinematic system by the natural injection $\mathcal{F}:V\rightarrow \mathfrak{X}(\mathcal{M})$.
The invariance of $V$ then implies that for all $A\in \mathcal{M}$ and $v\in V$ we have $d_{\star}L_{A}\mathcal{F}_{v}=\mathcal{F}_{w}$ for some $w\in V$. 
Define $\psi_{A}(v):=w$ and observer that this defines a left group action of G on $V=im\mathcal{F}\subset \mathfrak{X}(\mathcal{M})$ since $d_{\star}L$ is a left group action on $\mathfrak{X}(\mathcal{M})$.
The following lemma sums up this observation.
\begin{lemma}\cite{van2020equivariant}\quad
	A kinematic system $\mathcal{F}: V\rightarrow\mathfrak{X}(\mathcal{M})$ is left equivariant if and only if the vector subspace $im\mathcal{F}\subset\mathfrak{X}(\mathcal{M})$ is invariant under the left group action $d_{\star}L$.
\end{lemma}

The following lemma expresses equivariance in terms of an algebraic property of the lift $\Lambda$.
\begin{lemma}
	Let $\mathcal{F}:V\rightarrow \mathfrak{X}(\mathcal{M})$ be an equivariant kinematic system on a manifold $\mathcal{M}$ over a vector space $V$ with lift $\Lambda$. The lift $\Lambda$ satisfies
	\begin{equation}\label{equivairant_system_lift}
	Ad_{A^{-1}}\Lambda(AX,\psi_{A}(v))=\Lambda(X,v)
	\end{equation}
	for all $A,X\in \mathcal{M}$ and $v\in V$.
\end{lemma}
\begin{proof}
	We have for $A,X\in \mathcal{M}$ and $v\in V$,
	\begin{equation}\label{proof_equivariant_lift}
	\begin{aligned}
	Ad_{A^{-1}}\Lambda(AX,\psi_{A}(v))&=Ad_{A^{-1}}\mathcal{F}_{\psi_{A}(v)}(AX)((AX)^{-1})\\
	&=dR_{A}dL_{A^{-1}}\mathcal{F}_{\psi_{A}(v)}(AX)dL_{X^{-1}}dL_{A^{-1}}\\
	&=dL_{A^{-1}}\mathcal{F}_{\psi_{A}(v)}(L_{A}X)dL_{X^{-1}}\\
	&=dL_{A^{-1}} dL_{A}\mathcal{F}_{v}dL_{X^{-1}}\\
	&=\mathcal{F}_{v}X^{-1}=\Lambda(X,v)
	\end{aligned}
	\end{equation}
\end{proof}
\section{Full Second Order Kinematic Systems for Inertial-Integrated Navigation}
\subsection{Full Second Order Kinematic System in navigation frame}
The velocity vector $v_{eb}^n$, position vector $r_{eb}^n$, and attitude matrix $C_b^n$ can formula as the element of the $SE_2(3)$ matrix Lie group, that is
\begin{equation}\label{matrix_Lie_group_eb_n}
\mathcal{X}=\begin{bmatrix}
C_b^n & v_{eb}^n & r_{eb}^n\\
0_{1\times3} & 1 & 0\\
0_{1\times 3} & 0& 1
\end{bmatrix}\in SE_2(3)
\end{equation}

The inverse of the element can be written as follows
\begin{equation}\label{inverse_matrix_Lie_group_eb_b}
\mathcal{X}^{-1}=\begin{bmatrix}
C_n^b & -C_n^bv_{eb}^n & -C_n^br_{eb}^n\\
0_{1\times 3} &1 &0\\
0_{1\times 3} &0&1
\end{bmatrix}=\begin{bmatrix}
C_n^b & -v_{eb}^b & -r_{eb}^b\\
0_{1\times 3} &1 &0\\
0_{1\times 3} &0&1
\end{bmatrix}\in SE_2(3)
\end{equation}

Therefore, the differential equation of the $\mathcal{X}$ can be calculated as
\begin{equation}\label{differential}
\begin{aligned}
&\frac{d}{dt}\mathcal{X}=f_{u_t}(\mathcal{X})=\frac{d}{dt}\begin{bmatrix}
C_b^n & v_{eb}^n & r_{eb}^n\\
0_{1\times3} & 1 & 0\\
0_{1\times 3} & 0& 1
\end{bmatrix}
=\begin{bmatrix}
\dot{C}_b^n & \dot{v}_{eb}^n & \dot{r}_{eb}^n\\
0_{1\times3} & 0 & 0\\
0_{1\times 3} & 0& 0
\end{bmatrix}\\
=&\begin{bmatrix}
C_b^n(\omega_{ib}^b\times)-(\omega_{in}^n\times)C_b^n & C_b^nf_{ib}^b-\left[ (2\omega_{ie}^n+\omega_{en}^n)\times\right]v_{eb}^n+g_{ib}^n & -\omega_{en}^n\times r_{eb}^n+v_{eb}^n\\
0_{1\times3} & 0 & 0\\
0_{1\times 3} & 0& 0
\end{bmatrix}\\
\triangleq& \mathcal{X}W_1+W_2\mathcal{X}
\end{aligned}
\end{equation}
where $u_t$ is a sequence of inputs; $W_1$ and $W_2$ are denoted as
\begin{equation}\label{W_1_W_2}
W_1=\begin{bmatrix}
\omega_{ib}^b\times & f_{ib}^b & 0\\
0_{1\times3} & 0 & 0\\
0_{1\times 3} & 0& 0
\end{bmatrix},W_2=\begin{bmatrix}
-\omega_{in}^n\times & g_{ib}^n-\omega_{ie}^n\times v_{eb}^n & v_{eb}^n+\omega_{ie}^n\times r_{eb}^n\\
0_{1\times3} & 0 & 0\\
0_{1\times 3} & 0& 0
\end{bmatrix}
\end{equation}

It is easy to verify that the dynamical equation $f_{u_t}(\mathcal{X})$ is group-affine and the group-affine system owns the log-linear property of the corresponding error propagation~\cite{barrau2017the}:
\begin{equation}\label{proof_invariance}
\begin{aligned}
&f_{u_t}(\mathcal{X}_A)\mathcal{X}_B+\mathcal{X}_Af_{u_t}(\mathcal{X}_B)-\mathcal{X}_Af_{u_t}(I_d)\mathcal{X}_B\\
=&(\mathcal{X}_AW_1+W_2\mathcal{X}_A)\mathcal{X}_B+\mathcal{X}_A(\mathcal{X}_BW_1+W_2\mathcal{X}_B)-\mathcal{X}_A(W_1+W_2)\mathcal{X}_B\\
=&\mathcal{X}_A\mathcal{X}_BW_1+W_2\mathcal{X}_A\mathcal{X}_B\triangleq f_{u_t}(\mathcal{X}_A\mathcal{X}_B)
\end{aligned}
\end{equation}
\subsection{Full Second Order Kinematic System with transformed INS Mechanization in navigation frame}
The velocity vector $v_{eb}^n$, position vector $r_{eb}^n$, and attitude matrix $C_b^n$ can formula as the element of the $SE_2(3)$ matrix Lie group.
When the state defined on the matrix Lie group is given as
\begin{equation}\label{new_navigation_frame}
\mathcal{X}=\begin{bmatrix}
C_b^n & v_{ib}^n & r_{ib}^n\\
0_{1\times 3} & 1 & 0\\
0_{1\times 3} & 0 &1
\end{bmatrix}
\end{equation}
where $C_b^n$ is the direction cosine matrix from the body frame to the navigation frame; $v_{ib}^n$ is the velocity of body relative to the ECI frame expressed in the navigation frame.
Meanwhile, $v_{ib}^n=v_{eb}^n+\omega_{ie}^n\times r_{eb}^n$ and $r_{ib}^n=r_{eb}^n$.

Combine with the "plumb-bob gravity" equation in navigation frame, i.e., $g_{ib}^n=G_{ib}^n-(\omega_{ie}^n\times)^2r_{eb}^n$, the dynamic equation for the state $\mathcal{X}$ can be deduced as follows
\begin{equation}\label{group_affine_property_navigation_frame}
\begin{aligned}
&\frac{d}{dt}\mathcal{X}=f_{u_t}(\mathcal{X})=\frac{d}{dt}\begin{bmatrix}
C_b^n & v_{ib}^n & r_{ib}^n\\
0_{1\times 3} & 1 & 0\\
0_{1\times 3} & 0 &1
\end{bmatrix}=\begin{bmatrix}
\dot{C}_b^n & \dot{v}_{ib}^n & \dot{r}_{ib}^n\\
0_{1\times 3} & 0 & 0\\
0_{1\times 3} & 0 &0
\end{bmatrix}\\
=&\begin{bmatrix}
C_b^n(\omega_{ib}^b\times)-(\omega_{in}^n\times)C_b^n &-\omega_{in}^n\times v_{ib}^n+C_b^n f_{ib}^b+G_{ib}^n & -\omega_{in}^n\times r_{ib}^n+v_{ib}^n\\
0_{1\times 3} & 0 & 0\\
0_{1\times 3} & 0 &0
\end{bmatrix}\\
=&\begin{bmatrix}
C_b^n & v_{ib}^n & r_{ib}^n\\
0_{1\times 3} & 1 & 0\\
0_{1\times 3} & 0 &1
\end{bmatrix}\begin{bmatrix}
\omega_{ib}^b\times & f_{ib}^b & 0_{3\times 1}\\
0_{1\times 3} & 0 & 0\\
0_{1\times 3} & 0 &0
\end{bmatrix}+\begin{bmatrix}
-\omega_{in}^n\times & G_{ib}^n & v_{ib}^n\\
0_{1\times 3} & 0 & 0\\
0_{1\times 3} & 0 &0
\end{bmatrix}\begin{bmatrix}
C_b^n & v_{ib}^n & r_{ib}^n\\
0_{1\times 3} & 1 & 0\\
0_{1\times 3} & 0 &1
\end{bmatrix}\\
=&\mathcal{X}W_1+W_2\mathcal{X}
\end{aligned}
\end{equation}

It is easy to verify that the dynamical equation $f_{u_t}(\mathcal{X})$ is group-affine and the group-affine system owns the log-linear property of the corresponding error propagation~\cite{barrau2017the,luo2021se23}.

\subsection{Full Second Order Kinematic System in earth frame}
When the system  state is defined as 
\begin{equation}\label{new_state_eb_e}
\mathcal{X}=\begin{bmatrix}
C_b^e & v_{eb}^e & r_{eb}^e\\
0_{1\times 3} & 1 & 0\\
0_{1\times 3} & 0& 1
\end{bmatrix}\in SE_2(3)
\end{equation}
where $C_b^e$ is the direction cosine matrix from the body frame to the ECEF frame; $v_{eb}^e$ is the velocity of body frame relative to the ECEF frame expressed in the ECEF frame; $r_{eb}^e$ is the position of body frame relative to the ECEF frame expressed in the ECEF frame.

Then the dynamic equation of the state $\mathcal{X}$ can be deduced as follows
\begin{equation}\label{lie_group_state_differential_eb_e}
\begin{aligned}
&\frac{d}{dt}\mathcal{X}=f_{u_t}(\mathcal{X})=\frac{d}{dt}\begin{bmatrix}
C_b^e & v_{eb}^e & r_{eb}^e\\0_{1\times 3}&1&0\\0_{1\times 3}&0&1
\end{bmatrix}=\begin{bmatrix}
\dot{C}_b^e & \dot{v}_{eb}^e & \dot{r}_{eb}^e\\0_{1\times 3}&0&0\\0_{1\times 3}&0&0
\end{bmatrix}\\
=&\begin{bmatrix}
C_b^e(\omega_{ib}^b\times)-(\omega_{ie}^e\times)C_b^e & (-2\omega_{ie}^e\times)v_{eb}^e+C_b^ef^b+g_{ib}^e & v_{eb}^e\\
0_{1\times 3}&0&0\\0_{1\times 3}&0&0
\end{bmatrix}\\
=&\begin{bmatrix}
C_b^e & v_{eb}^e & r_{eb}^e\\0_{1\times 3}&1&0\\0_{1\times 3}&0&1
\end{bmatrix}\begin{bmatrix}
\omega_{ib}^b\times & f^b & 0_{3\times 1}\\ 0_{1\times 3} &0&0\\ 0_{1\times 3}&0&0
\end{bmatrix}+\\
&\begin{bmatrix}
-\omega_{ie}^e\times & g_{ib}^e-\omega_{ie}^e\times v_{eb}^e & v_{eb}^e+\omega_{ie}^e\times r_{eb}^e\\ 0_{1\times 3} &0&0\\ 0_{1\times 3}&0&0
\end{bmatrix}\begin{bmatrix}
C_b^e & v_{eb}^e & r_{eb}^e\\0_{1\times 3}&1&0\\0_{1\times 3}&0&1
\end{bmatrix}
\triangleq \mathcal{X}W_1+W_2\mathcal{X}
\end{aligned}
\end{equation}

It is easy to verify that the dynamical equation is group-affine property. 
\subsection{Full Second Order Kinematic System with transformed INS Mechanization in earth frame}
When the system state is defined as
\begin{equation}\label{new_state_ie_n}
\mathcal{X}=\begin{bmatrix}
C_b^e & v_{ib}^e & r_{ib}^e\\
0_{1\times3} & 1 & 0\\
0_{1\times 3} & 0& 1
\end{bmatrix}\in SE_2(3)
\end{equation}
where $C_b^e$ is the direction cosine matrix from the body frame to the ECEF frame; $v_{ib}^e$ is the velocity of body frame relative to the ECI frame expressed in the ECEF frame; $r_{ib}^e$ is the position of body frame relative to the ECI frame expressed in the ECEF frame.
Meanwhile, $v_{ib}^e=v_{eb}^e+\omega_{ie}^e\times r_{eb}^e$ and $r_{ib}^e=r_{eb}^e$.

Firstly, the state and its inverse defined on the matrix Lie group are given as
\begin{equation}\label{matrix_lie_group}
\mathcal{X}=\begin{bmatrix}
C_b^e & v_{ib}^e & r_{ib}^e\\0_{1\times 3}&1&0\\0_{1\times 3}&0&1
\end{bmatrix},\mathcal{X}^{-1}=\begin{bmatrix}
C_e^b & -v_{ib}^b & -r_{ib}^b\\0_{1\times 3}&1&0\\0_{1\times 3}&0&1
\end{bmatrix}
\end{equation}

Then, combine with the "plumb-bob gravity" equation in earth frame: $g_{ib}^e=G_{ib}^e-(\omega_{ie}^e\times)^2r_{eb}^e$, the dynamic equation of the state $\mathcal{X}$  can be deduced as follows
\begin{equation}\label{lie_group_state_differential}
\begin{aligned}
&\frac{d}{dt}\mathcal{X}=f_{u_t}(\mathcal{X})=\frac{d}{dt}\begin{bmatrix}
C_b^e & v_{ib}^e & r_{ib}^e\\0_{1\times 3}&1&0\\0_{1\times 3}&0&1
\end{bmatrix}=\begin{bmatrix}
\dot{C}_b^e & \dot{v}_{ib}^e & \dot{r}_{ib}^e\\0_{1\times 3}&0&0\\0_{1\times 3}&0&0
\end{bmatrix}\\
=&\begin{bmatrix}
C_b^e(\omega_{ib}^b\times)-(\omega_{ie}^e\times)C_b^e & (-\omega_{ie}^e\times)v_{ib}^e+C_b^ef_{ib}^b+G_{ib}^e & (-\omega_{ie}^e\times)r_{ib}^e+v_{ib}^e\\
0_{1\times 3}&0&0\\0_{1\times 3}&0&0
\end{bmatrix}\\
=&\begin{bmatrix}
C_b^e & v_{ib}^e & r_{ib}^e\\0_{1\times 3}&1&0\\0_{1\times 3}&0&1
\end{bmatrix}\begin{bmatrix}
\omega_{ib}^b\times & f_{ib}^b & 0_{3\times 1}\\ 0_{1\times 3} &0&0\\ 0_{1\times 3}&0&0
\end{bmatrix}+\begin{bmatrix}
-\omega_{ie}^e\times & G_{ib}^e & v_{ib}^e\\ 0_{1\times 3} &0&0\\ 0_{1\times 3}&0&0
\end{bmatrix}\begin{bmatrix}
C_b^e & v_{ib}^e & r_{ib}^e\\0_{1\times 3}&1&0\\0_{1\times 3}&0&1
\end{bmatrix}\\
=&\mathcal{X}W_1+W_2\mathcal{X}
\end{aligned}
\end{equation}

It is easy to verify that the dynamical equation(\ref{lie_group_state_differential}) satisfies the group-affine property.
\section{Equivariant Filtering for Second Order Kinematic Systems on $\mathrm{T}\mathbf{SE}_2(3)$}
The second order kinematic systems on $\mathrm{T}\mathbf{SE}_2(3)$ can be formulated as
\begin{equation}\label{kinematics_inertial}
\mathcal{F}_{v}(X)=\mathcal{F}_{W_1,W_2}(X)=XW_1+W_2X
\end{equation}

Define the function $\Lambda:\mathcal{M}\times V\rightarrow \mathfrak{g}$ as
\begin{equation}\label{Lift_inertial_navigation}
\Lambda(X,v)=\Lambda(X,(W_1,W_2)):=XW_1X^{-1}+W_2
\end{equation}
It is easy to verify that $\Lambda$ is a lift according to equation (\ref{lift}).

Define the velocity action $\psi:G\times V\rightarrow V$ by 
\begin{equation}\label{velocity_action}
\psi_A(v)=\psi(A,v)=\psi(A,(W_1,W_2)):=(W_1,Ad_{A}W_2)
\end{equation}
It is easily verified that $\psi$ is a group action on V.

To verify that $\Lambda$ is equivariant (\ref{equivairant_system_lift}), compute
\begin{equation}\label{inertial_equivariant_proof}
\begin{aligned}
Ad_{A^{-1}}\Lambda(AX,\psi_{A}(v))&=Ad_{A^{-1}}\Lambda(AX,(W_1,Ad_{A}W_2))=Ad_{A^{-1}}\left(AXW_1(AX)^{-1}+Ad_{A}W_2 \right)\\
&=XW_1X^{-1}+W_2=\Lambda(X,v)
\end{aligned}
\end{equation}
where $\phi:G\times\mathcal{M}\rightarrow \mathcal{M}$ is defined as $\phi(A,X):=AX\in \mathcal{M}$.
It is straightforward to verify that $\phi$ is a smooth, transitive left group action of $G$ on $\mathcal{M}$. 

In general, the group $G$ and $\mathcal{M}$ are the matrix Lie group $SE_2(3)$ in the inertial-integrated navigation system. 
More specifically, the state can be translated by the left group action between four states.
For the full states mentioned in equation (\ref{matrix_Lie_group_eb_n}), equation (\ref{new_navigation_frame}), equation (\ref{new_state_eb_e}), and equation (\ref{new_state_ie_n}), $A$ can be defined as 
\begin{equation}\label{A}
A_1=\begin{bmatrix}
C_b^e & 0_{3\times 1}& 0_{3\times 1}\\
0_{1\times 3}&1&0\\
0_{1\times 3}&0&1
\end{bmatrix},A_2=\begin{bmatrix}
I_{3\times 3} & \omega_{ie}^n\times r_{eb}^n & 0_{3\times 1}\\
0_{1\times 3}&1&0\\
0_{1\times 3}&0&1
\end{bmatrix},A_3=\begin{bmatrix}
I_{3\times 3} & \omega_{ie}^e\times r_{eb}^e & 0_{3\times 1}\\
0_{1\times 3}&1&0\\
0_{1\times 3}&0&1
\end{bmatrix}
\end{equation}

\section{Equivariant Filtering Design for the Inertial-Integrated Navigation}
As the second order kinematics are shown to be equivariant, the equivariant filtering can be designed for the inertial-integrated navigation systems. 
Therefore, there are four kinematics can be used to design the equivariant filtering algorithm respectively.
To save space, only one of them is given here, and the rest can be derived similarly.
Moreover, the error on the Lie group can be defined as one of $\tilde{X}^{-1}X$, $X^{-1}\tilde{X}$, $\tilde{X}X^{-1}$, and $X\tilde{X}^{-1}$, where the first two are left invariant and the last two are right invariant.
In the end, we only choose the right invariant error $\tilde{X}X^{-1}$ with second order kinematic with transformed INS mechanization in earth frame. More algorithms can be deduced by combining a kinematic system and a invariant error.

The right invariant error is defined as
\begin{equation}\label{key}
\begin{aligned}
\eta^R&=\tilde{\mathcal{X}}\mathcal{X}^{-1}=
\begin{bmatrix}
\tilde{C}_b^e & \tilde{v}_{ib}^e & \tilde{r}_{ib}^e\\0_{1\times 3}&1&0\\0_{1\times 3}&0&1
\end{bmatrix}
\begin{bmatrix}
C_e^b & -v_{ib}^b & -r_{ib}^b\\0_{1\times 3}&1&0\\0_{1\times 3}&0&1
\end{bmatrix}
=\begin{bmatrix}
\tilde{C}_b^eC_e^b & \tilde{v}_{ib}^e -\tilde{C}_b^ev_{ib}^b& \tilde{r}_{ib}^e-\tilde{C}_b^er_{ib}^b\\
0_{1\times 3}&1&0\\0_{1\times 3}&0&1
\end{bmatrix}
\end{aligned}
\end{equation}

Similarity, the new error state defined on the matrix Lie group $SE_2(3)$ can be denoted as
\begin{equation}\label{attitude_velocity_position_right}
\begin{aligned}
\tilde{C}_b^eC_e^b= C_{e}^{e'}=&\exp_G(\phi^e\times)\approx I-\phi^e\times\\
\eta_v^R=J\rho_{v}^e=&\tilde{v}_{ib}^e -\tilde{C}_b^ev_{ib}^b=\tilde{v}_{ib}^e- v_{ib}^e+v_{ib}^e-\tilde{C}_b^eC_e^bv_{ib}^e=\delta v_{ib}^e+(I-\exp_G(\phi^e\times)) v_{ib}^e\\
=&\delta v_{ib}^e+\phi^e\times v_{ib}^e\\ 
\eta_r^R=J\rho_{r}^e=&\tilde{r}_{ib}^e-\tilde{C}_b^er_{ib}^b=\tilde{r}_{ib}^e- r_{ib}^e+r_{ib}^e-\tilde{C}_b^eC_e^br_{ib}^e=\delta r_{ib}^e+(I-\exp_G(\phi^e\times)) r_{ib}^e\\
=&\delta r_{ib}^e+\phi^e\times r_{ib}^e
\end{aligned}
\end{equation}
where $e'$ is the estimated earth frame.

Meanwhile, the right invariant error satisfies that
\begin{equation}\label{lie_algebra_right_error}
\eta^R=\begin{bmatrix}
\exp_G(\phi^e\times) & J\rho_{v}^e & J\rho_{r}^e\\
0_{1\times 3} & 1& 0\\ 0_{1\times 3}&0&1
\end{bmatrix}=\exp_G\left( \begin{bmatrix}
(\phi^e\times) & \rho_{v}^e & \rho_{r}^e\\
0_{1\times 3} & 0& 0\\ 0_{1\times 3}&0&0
\end{bmatrix}\right)=\exp_G\left( \Lambda\begin{bmatrix}
\phi^e \\ \rho_{v}^e \\ \rho_{r}^e
\end{bmatrix} \right)=\exp_G(\Lambda(\rho^e))
\end{equation}
where $\phi^e$ is the attitude error state; $J\rho_v^e$ is the new definition of velocity error state; $J\rho_r^e$ is the new definition of position error state; $\Lambda$ is a linear isomorphism between the Lie algebra and the associated vector; $\phi^e\times$ denotes the skew-symmetric matrix generated from a 3D vector $\phi^e\in \mathbb{R}^3$; $\exp_G$ denotes the matrix exponential mapping; $J$ is the left Jacobian matrix of the 3D orthogonal rotation matrices group $SO(3)$ which is given as
\begin{equation}\label{left_Jacobian}
J=J_l(\phi)=\sum_{n=0}^{\infty}\frac{1}{(n+1)!}(\phi_{\wedge})^n=I_3+\frac{1-\cos\theta}{\theta^2}\phi_{\wedge}+\frac{\theta-\sin\theta}{\theta^3}\phi_{\wedge}^2,\theta=||\phi||
\end{equation}
$\rho^e=\begin{bmatrix}
{\phi^e}^T& {\rho_{v}^e}^T& {\rho_{r}^e}^T
\end{bmatrix}^T$ is a 9-dimensional state error vector defined on the Euclidean space that corresponding to the error state $\eta^R$ which is defined on the matrix Lie group. $J$ can be approximated as $J\approx I_{3\times 3}$ if $||\phi^e||$ is small enough.

$\exp_G(\phi^e\times)$ is the Rodriguez formula of the rotation vector and can be calculated by
\begin{equation}\label{Rodriguez_e}
\exp_G(\phi^e\times)=\cos\phi I_{3\times3}+\frac{1-\cos\phi}{\phi^2}\phi^e{\phi^e}^T+\frac{\sin\phi}{\phi}(\phi^e\times),\phi=||\phi^e||
\end{equation}

\begin{remark}
	It is worth noting that the direction cosine matrix $C_b^e$ represents the rotation from $b$ frame to $e$ frame where $e$ frame is fixed. Therefore, the rotation $C_{e'}^e$ can be approximated as $C_{e'}^e\approx I+\phi^e\times$ but $C_{e}^{e'}$ can be approximated as $C_{e}^{e'}\approx I-\phi^e\times$ as we assum the $e$ frame is fixed. When the error state is defined as $\eta^R=\mathcal{X}\tilde{\mathcal{X}}^{-1}$, the attitude error state will be defined as $C_b^e\tilde{C}_e^b\approx C_{e'}^e\approx I+\phi^e\times$.
\end{remark}

The differential equation of the attitude error state is given as
\begin{equation}\label{attitude_error_right}
\begin{aligned}
&\frac{d}{dt}(\tilde{C}_b^eC_e^b)=\dot{\tilde{C}}_b^eC_e^b+\tilde{C}_b^e\dot{C}_e^b\\
=&\left[\tilde{C}_b^e(\tilde{\omega}_{ib}^b\times)-(\tilde{\omega}_{ie}^e\times)\tilde{C}_b^e\right]C_e^b+\tilde{C}_b^e\left[C_e^b(\omega_{ie}^e\times)-(\omega_{ib}^b\times)C_e^b\right]\\
=&\tilde{C}_b^e(\tilde{{\omega}}_{ib}^b\times)C_e^b-({\omega}_{ie}^e\times)\tilde{C}_b^eC_e^b+\tilde{C}_b^eC_e^b(\omega_{ie}^e\times)-\tilde{C}_b^e(\omega_{ib}^b\times)C_e^b\\
\approx &\tilde{C}_b^e(\delta{\omega}_{ib}^b\times)C_e^b-({\omega}_{ie}^e\times)(I-\phi^e\times)+(I-\phi^e\times)(\omega_{ie}^e\times)\\
=&\tilde{C}_b^e(\delta \omega_{ib}^b\times){\tilde{C}_e^b\tilde{C}_b^eC_e^b}+({\omega}_{ie}^e\times)(\phi^e\times)-(\phi^e\times)(\omega_{ie}^e\times)\\
\approx& ((\tilde{C}_b^e\delta \omega_{ib}^b)\times)(I-\phi^e\times)-((\phi^e\times\omega_{ie}^e)\times)\\
\approx&(\tilde{C}_b^e\delta \omega_{ib}^b)\times-(\phi^e\times\omega_{ie}^e)\times=(\tilde{C}_b^e(\delta b_g^b+w_g^b))\times-(\phi^e\times\omega_{ie}^e)\times
\end{aligned}
\end{equation}
where the second order small quantity $\left((\tilde{C}_b^e\delta \omega_{ib}^b)\times\right)(\phi^e\times)$ is neglected.
Therefore, the equation(\ref{attitude_error_right}) can be simplified as
\begin{equation}\label{attitude_error_differential_right}
\dot{\phi}^e=\phi^e\times\omega_{ie}^e-\tilde{C}_b^e\delta b_g^b-\tilde{C}_b^ew_g^b={-\omega_{ie}^e\times\phi^e}-\tilde{C}_b^e\delta b_g^b-\tilde{C}_b^ew_g^b
\end{equation}

The differential equation of the velocity error state is given as
\begin{equation}\label{new_velocity_error_right}
\begin{aligned}
&\frac{d}{dt}(J\rho_{v}^e)=\frac{d}{dt}(\tilde{v}_{ib}^e-\tilde{C}_b^eC_e^bv_{ib}^e)=\dot{\tilde{v}}_{ib}^e -\tilde{C}_b^eC_e^b\dot{v}_{ib}^e-\frac{d}{dt}(\tilde{C}_b^eC_e^b)v_{ib}^e\\
=&\left[(-\tilde{\omega}_{ie}^e\times)\tilde{v}_{ib}^e+\tilde{C}_b^e\tilde{f}^b+\tilde{G}_{ib}^e\right] -\tilde{C}_b^eC_e^b\left[(-\omega_{ie}^e\times)v_{ib}^e+C_b^ef^b+G_{ib}^e\right]\\
&-\left(\tilde{C}_b^e(\delta \omega_{ib}^b\times)\tilde{C}_e^b\tilde{C}_b^eC_e^b-({\omega}_{ie}^e\times)\tilde{C}_b^eC_e^b+\tilde{C}_b^eC_e^b(\omega_{ie}^e\times)\right)v_{ib}^e\\
=&{\tilde{C}_b^e\delta f^b}-({\omega}_{ie}^e\times)(\tilde{v}_{ib}^e-\tilde{C}_b^eC_e^bv_{ib}^e)
-\left((\tilde{C}_b^e\delta \omega_{ib}^b)\times\right){\tilde{C}_b^eC_e^b v_{ib}^e}+\tilde{G}_{ib}^e-\tilde{C}_b^eC_e^bG_{ib}^e\\
\approx&\tilde{C}_b^e\delta f^b-({\omega}_{ie}^e\times)J\rho_v^e
-(\tilde{C}_b^e\delta \omega_{ib}^b)\times {(\tilde{v}_{ib}^e-J\rho_v^e)}+\tilde{G}_{ib}^e-(I-\phi^e\times)G_{ib}^e\\
\approx&-G_{ib}^e\times \phi^e{-({\omega}_{ie}^e\times)J\rho_v^e}+ \tilde{v}_{ib}^e\times(\tilde{C}_b^e\delta \omega_{ib}^b)+\tilde{C}_b^e\delta f^b+\tilde{G}_{ib}^e-G_{ib}^e\\
=&-G_{ib}^e\times \phi^e{-({\omega}_{ie}^e\times)J\rho_v^e}+ \tilde{v}_{ib}^e\times(\tilde{C}_b^e(\delta b_g^b+w_g^b))+\tilde{C}_b^e(\delta b_a^b+w_a^b)+\tilde{G}_{ib}^e-G_{ib}^e
\end{aligned}
\end{equation}
where the second order small quantity $(J\rho_v^e\times)(\tilde{C}_b^e\delta \omega_{ib}^b)$ is neglected; and as $G_{ib}^e$ can be approximated as constant, $\tilde{G}_{ib}^e-G_{ib}^e$ can also be neglected.

In the same way, the differential equation of the position error state is given as
\begin{equation}\label{new_position_error_right}
\begin{aligned}
&\frac{d}{dt}(J\rho_{r}^e)=\frac{d}{dt}(\tilde{r}_{ib}^e-\tilde{C}_b^eC_e^br_{ib}^e)=\dot{\tilde{r}}_{ib}^e -\tilde{C}_b^eC_e^b\dot{r}_{ib}^e - \frac{d}{dt}(\tilde{C}_b^eC_e^b)r_{ib}^e\\
=&\left[(-\tilde{\omega}_{ie}^e\times)\tilde{r}_{ib}^e+\tilde{v}_{ib}^e\right] -\tilde{C}_b^eC_e^b\left[(-\omega_{ie}^e\times)r_{ib}^e+v_{ib}^e\right]\\
&-\left(\tilde{C}_b^e(\delta \omega_{ib}^b\times)\tilde{C}_e^b\tilde{C}_b^eC_e^b-({\omega}_{ie}^e\times)\tilde{C}_b^eC_e^b+\tilde{C}_b^eC_e^b(\omega_{ie}^e\times)\right)r_{ib}^e\\
\approx&(-\tilde{\omega}_{ie}^e\times)(\tilde{r}_{ib}^e-\tilde{C}_b^eC_e^br_{ib}^e)+(\tilde{v}_{ib}^e-\tilde{C}_b^eC_e^bv_{ib}^e)-((\tilde{C}_b^e\delta \omega_{ib}^b)\times){\tilde{C}_b^eC_e^br_{ib}^e}\\
\approx &{(-\tilde{\omega}_{ie}^e\times)J\rho_r^e}+J\rho_v^e+((\tilde{C}_b^e\delta \omega_{ib}^b)\times) {(\tilde{r}_{ib}^e-J\rho_r^e)}\\
\approx&{(-\tilde{\omega}_{ie}^e\times)J\rho_r^e}+J\rho_v^e+\tilde{r}_{ib}^e\times(\tilde{C}_b^e\delta \omega_{ib}^b)\\
=&{(-{\omega}_{ie}^e\times)J\rho_r^e}+J\rho_v^e+\tilde{r}_{ib}^e\times(\tilde{C}_b^e(\delta b_g^b+w_g^b))
\end{aligned}
\end{equation}
where the second order small quantity $(J\rho_r^e\times)(C_b^e\delta \omega_{ib}^b)$ is neglected.

Thus, the error state $\delta x$, the error state transition matrix $F$, and the noise driven matrix $G$ of the inertial-integrated error state dynamic equation for equivariant filtering with estimated earth frame attitude are represented as
\begin{equation}\label{state_x_right}
\delta x=\begin{bmatrix}
\phi^e\\ J\rho_{v}^e \\J\rho_{r}^e \\\delta b_g^b \\\delta b_a^b
\end{bmatrix}, F=\begin{bmatrix}
-\omega_{ie}^e\times & 0 & 0& -\tilde{C}_b^e &0\\
-G_{ib}^e\times &-\omega_{ie}^e\times&0 & \tilde{v}_{ib}^e\times \tilde{C}_b^e& \tilde{C}_b^e\\
0&I&-\omega_{ie}^e\times&\tilde{r}_{ib}^e\times \tilde{C}_b^e&0\\
0&0&0&0&0\\
0&0&0&0&0
\end{bmatrix},G=\begin{bmatrix}
-\tilde{C}_b^e&0\\ \tilde{v}_{ib}^e\times \tilde{C}_b^e &\tilde{C}_b^e\\ \tilde{r}_{ib}^e\times \tilde{C}_b^e&0\\0&0\\0&0
\end{bmatrix}
\end{equation}
\begin{remark}
	As the rotation of the earth is taken into consideration, the error state transition function is slightly different from the error state transition function derived in~\cite{hartley2020contact}. This is reasonable as the rotation of the earth is not sensed in the low cost inertial sensor for applications such as inertial assisted SLAM.
\end{remark}

The right measurement Jacobian is in the form of·\cite{luo2021se23}
\begin{equation}\label{measurement_matrix_right}
H_r=\begin{bmatrix}
-(\tilde{r}_{ib}^e+C_b^el^b) \times & 0& I &0&0
\end{bmatrix}
\end{equation}

The feedback correction is corresponding to the error state definition in equation (\ref{attitude_velocity_position_right}) and can be 
termed as the inverse process. The feedback of attitude, velocity, and position can be easily obtained 
\begin{equation}\label{feedback_attitude_velocity_position}
\begin{aligned}
C_b^e&=\exp_G(\phi^e\times)\tilde{C}_b^e\\
v_{ib}^e&=\tilde{v}_{ib}^e-J \rho_v^e-\tilde{v}_{ib}^e\times \phi^e\\
r_{ib}^e&=\tilde{r}_{ib}^e-J \rho_r^e-\tilde{r}_{ib}^e\times  \phi^e
\end{aligned}
\end{equation}

\begin{remark}
	If the right invariant error is defined as $\eta^R=\mathcal{X}\tilde{\mathcal{X}}^{-1}$, the results are same.
\end{remark}

The left invariant error is defined as
\begin{equation}\label{left_invariant_error_estimated}
\eta^L=\tilde{\mathcal{X}}^{-1}\mathcal{X}=\begin{bmatrix}
\tilde{C}_e^b & -\tilde{v}_{ib}^b & -\tilde{r}_{ib}^b
\\ 0_{1\times 3}&1&0\\0_{1\times 3}&0&1
\end{bmatrix}\begin{bmatrix}
{C}_b^e & {v}_{ib}^e & {r}_{ib}^e\\
0_{1\times 3}&1&0\\0_{1\times 3}&0&1
\end{bmatrix}=\begin{bmatrix}
\tilde{C}_e^bC_b^e & \tilde{C}_e^bv_{ib}^e -\tilde{v}_{ib}^b& \tilde{C}_e^b{r}_{ib}^e-\tilde{r}_{ib}^b\\
0_{1\times 3}&1&0\\0_{1\times 3}&0&1
\end{bmatrix}
\end{equation}

Then the error defined on the Lie group can be converted to the Lie algebra and the Euclidean space. The new defined error state of attitude, velocity, and position can be defined as
\begin{equation}\label{attitude_velocity_position_estimated}
\begin{aligned}
\tilde{C}_e^b{C}_b^e =&\exp_G(\phi^b\times)\approx I+\phi^b\times\\
\eta_v^L=J\rho_{v}^b=&\tilde{C}_e^b{v}_{ib}^e -\tilde{v}_{ib}^b=\tilde{C}_e^b{v}_{ib}^e-\tilde{C}_e^b\tilde{v}_{ib}^e=\tilde{C}_e^b({v}_{ib}^e -\tilde{v}_{ib}^e)=-\tilde{C}_e^b\delta v_{ib}^e \\ 
\eta_r^L=J\rho_{r}^b=&\tilde{C}_e^b{r}_{ib}^e-\tilde{r}_{ib}^b =\tilde{C}_e^b{r}_{ib}^e-\tilde{C}_e^b\tilde{r}_{ib}^e=\tilde{C}_e^b({r}_{ib}^e -\tilde{r}_{ib}^e)=-\tilde{C}_e^b\delta r_{ib}^e
\end{aligned}
\end{equation}
where $\phi^b$ is the attitude error state, $J\rho_{v}^b$ is the new definition of velocity error state, $J\rho_{r}^b$ is the new definition of position error state;
$J$ is the left Jacobian matrix given in equation(\ref{left_Jacobian}).
The errors can be termed as common frame representation in the body frame~\cite{andrle2015attitude}.
The left-invariant error also satisfies that
\begin{equation}\label{lie_algebra_left_error_estimated}
\eta^L=\begin{bmatrix}
\exp_G(\phi^b\times) & J\rho_{v}^b & J\rho_{r}^b\\
0_{1\times 3} & 1& 0\\ 0_{1\times 3}&0&1
\end{bmatrix}=\exp_G\left(\begin{bmatrix}
(\phi^b)\times & \rho_{v}^b & \rho_{r}^b\\0_{1\times 3}&0&0\\0_{1\times 3}&0&0
\end{bmatrix} \right)=\exp_G\left(\Lambda\begin{bmatrix}
\phi^b \\ \rho_{v}^b \\ \rho_{r}^b
\end{bmatrix} \right)
\end{equation}

The differential equation of the attitude error state is given as
\begin{equation}\label{attitude_error_estimated}
\begin{aligned}
\frac{d}{dt}(\tilde{C}_e^b{C}_b^e)&=\dot{\tilde{C}}_e^b{C}_b^e+\tilde{C}_e^b\dot{{C}}_b^e\\
&=\left[\tilde{C}_e^b(\tilde{\omega}_{ie}^e\times)-(\tilde{\omega}_{ib}^b\times)\tilde{C}_e^b\right]{C}_b^e+\tilde{C}_e^b\left[{C}_b^e({\omega}_{ib}^b\times)-({\omega}_{ie}^e\times){C}_b^e\right]\\
&=\tilde{C}_e^b(\omega_{ie}^e\times){C}_b^e-(\tilde{\omega}_{ib}^b\times)\tilde{C}_e^b{C}_b^e+\tilde{C}_e^b{C}_b^e({\omega}_{ib}^b\times)-\tilde{C}_e^b({\omega}_{ie}^e\times){C}_b^e\\
&\approx -(\tilde{\omega}_{ib}^b\times)(I+\phi^b\times)+(I+\phi^b\times)((\tilde{\omega}_{ib}^b-\delta \omega_{ib}^b)\times)\\
&=-(\tilde{\omega}_{ib}^b\times)(\phi^b\times)-(\delta \omega_{ib}^b)\times+ (\phi^b\times)(\tilde{\omega}_{ib}^b\times)-\phi^b\times(\delta \omega_{ib}^b)\times \\
&\approx (\phi^b\times \tilde{\omega}_{ib}^b)\times-\delta \omega_{ib}^b\times= (\phi^b\times \tilde{\omega}_{ib}^b)\times-(\delta b_g^b+w_g^b)\times
\end{aligned}
\end{equation}
where the angular velocity error of the earth's rotation can be neglected, i.e., $\tilde{\omega}_{ie}^e=\omega_{ie}^e$; and second order small quantity $(\phi^b\times)(\delta \omega_{ib}^b\times)$ is also neglected. Therefore, the equation(\ref{attitude_error_estimated}) can be simplified as
\begin{equation}\label{attitude_error_differential_estimated}
\dot{\phi}^b=\phi^b\times \tilde{\omega}_{ib}^b-\delta \omega_{ib}^b=-\tilde{\omega}_{ib}^b \times \phi^b-\delta \omega_{ib}^b=-\tilde{\omega}_{ib}^b\times\phi^b-\delta b_g^b-w_g^b
\end{equation}

The differential equation of the velocity error state is given as
\begin{equation}\label{new_velocity_error_estimated}
\begin{aligned}
&\frac{d}{dt}(J\rho_{v}^b)=-\dot{\tilde{C}}_e^b\delta v_{ib}^e+\tilde{C}_e^b(\dot{{v}}_{ib}^e-\dot{\tilde{v}}_{ib}^e)\\
=&-\left[\tilde{C}_e^b(\tilde{\omega}_{ie}^e\times)-(\tilde{\omega}_{ib}^b\times)\tilde{C}_e^b\right]\delta v_{ib}^e\\
&+\tilde{C}_e^b\left(\left[(-\omega_{ie}^e\times)v_{ib}^e+C_b^ef^b+G_{ib}^e\right]- \left[(-\tilde{\omega}_{ie}^e\times)\tilde{v}_{ib}^e+\tilde{C}_b^e\tilde{f}^b+\tilde{G}_{ib}^e\right]\right)\\
=&-\tilde{C}_e^b(\omega_{ie}^e\times)\delta v_{ib}^e+(\tilde{\omega}_{ib}^b\times){red}{\tilde{C}_e^b\delta v_{ib}^e}+\tilde{C}_e^b{C}_b^e{f}^b-\tilde{f}^b-\tilde{C}_e^b\omega_{ie}^e\times({v}_{ib}^e-\tilde{v}_{ib}^e)\\
&+\tilde{C}_e^b({G}_{ib}^e-\tilde{G}_{ib}^e)\\
\approx&-(\tilde{\omega}_{ib}^b\times)J\rho_{v}^b+(I+\phi^b\times)(\tilde{f}^b-\delta b_a^b-w_a^b)-\tilde{f}^b+\tilde{C}_e^b({G}_{ib}^e-\tilde{G}_{ib}^e)\\
=&-(\tilde{\omega}_{ib}^b\times)J\rho_{v}^b+\phi^b\times \tilde{f}^b-\phi^b\times\delta f^b +\tilde{C}_e^b({G}_{ib}^e-\tilde{G}_{ib}^e)-\delta f^b \\
\approx & -(\omega_{ib}^b\times)J\rho_{v}^b-f^b\times \phi^b+\tilde{C}_e^b({G}_{ib}^e-\tilde{G}_{ib}^e)+\delta f^b\\
= & -(\omega_{ib}^b\times)J\rho_{v}^b-f^b\times \phi^b+\tilde{C}_e^b({G}_{ib}^e-\tilde{G}_{ib}^e)-\delta b_a^b-w_a^b
\end{aligned}
\end{equation}
where the second order small quantity $\phi^b\times\delta f^b $ is neglected; and as $G_{ib}^e$ can be approximated as constant, $\tilde{C}_e^b({G}_{ib}^e-\tilde{G}_{ib}^e)$ can also be neglected.

In the same way, the differential equation of the position error state is given as
\begin{equation}\label{new_position_error_estimated}
\begin{aligned}
&\frac{d}{dt}(J\rho_{r}^b)=-\dot{\tilde{C}}_e^b\delta r_{ib}^e+\tilde{C}_e^b(\dot{{r}}_{ib}^e-\dot{\tilde{r}}_{ib}^e)\\
=&-\left[\tilde{C}_e^b(\omega_{ie}^e\times)-(\tilde{\omega}_{ib}^b\times)\tilde{C}_e^b\right]\delta r_{ib}^e
+\tilde{C}_e^b\left(\left[(-\omega_{ie}^e\times)r_{ib}^e+v_{ib}^e\right]- \left[(-\tilde{\omega}_{ie}^e\times)\tilde{r}_{ib}^e+\tilde{v}_{ib}^e\right]\right)\\
=&-\tilde{C}_e^b(\omega_{ie}^e\times)\delta r_{ib}^e+(\tilde{\omega}_{ib}^b\times){red}{\tilde{C}_e^b\delta r_{ib}^e}-\tilde{C}_e^b\omega_{ie}^e\times({r}_{ib}^e-\tilde{r}_{ib}^e)+\tilde{C}_e^b({v}_{ib}^e-\tilde{v}_{ib}^e)\\
=&-\tilde{\omega}_{ib}^b\times J\rho_{r}^b+J\rho_{v}^b
\end{aligned}
\end{equation}

Thus, the inertial-integrated error state dynamic equation for the $SE_2(3)$ based EKF can be obtained 
\begin{equation}\label{invariant_ekf_estimated}
\delta\dot{x}=F\delta x+Gw
\end{equation}
where $F$ is the error state transition matrix; $\delta x$ is the error state including the terms about bias; G is the noise driven matrix. Their definition is given as
\begin{equation}\label{state_x_estimated}
\begin{aligned}
&\delta x=\begin{bmatrix}
\phi^b\\ J\rho_{v}^b \\J\rho_{r}^b \\\delta b_g^b \\\delta b_a^b
\end{bmatrix}, 
F=\begin{bmatrix}
-\tilde{\omega}_{ib}^b\times & 0 & 0& -I_{3\times 3} &0\\
-\tilde{f}^b\times &-\tilde{\omega}_{ib}^b\times&0 & 0& -I_{3\times 3}\\
0&I_{3\times 3}&-\tilde{\omega}_{ib}^b\times&0&0\\
0&0&0&0&0\\
0&0&0&0&0
\end{bmatrix},\\
&
G=\begin{bmatrix}
-I_{3\times 3}&0&0&0\\0&-I_{3\times 3}&0&0\\0&0&0&0\\0&0&I_{3\times 3}&0\\0&0&0&I_{3\times 3}
\end{bmatrix},
w=\begin{bmatrix}w_g^b\\w_a^b \\ w_{b_g}^b \\ w_{b_a}^b\end{bmatrix}
\end{aligned}
\end{equation}

The left measurement Jacobian is in the form of~\cite{luo2021se23}
\begin{equation}\label{measurement_matrix_left}
H_l=\begin{bmatrix}
-C_b^e(l^b\times) & 0& {C}_{b}^e &0&0
\end{bmatrix}
\end{equation} 
\section{Analytic State Transition Matrix}
It has been proved that the observability of the Invariant EKF (InEKF) is consistent with the nonlinear system therein~\cite{hartley2020contact}. 
Thus, the inconsistency caused by the linearization points in the IEKF-based algorithm can be avoided in the traditional EKF.
The observability of the equivalent filtering in the transformed earth frame is the same as that of IEKF when the error is left invariant~\cite{hartley2020contact}. However, the rotation of the earth is not taken into consideration when the error is right invariant and the detailed derivations of the left invariant error and right invariant error are not given in \cite{hartley2020contact}.
Therefore, we give the detailed derivations both of the left invariant error and right invariant error for the equivariant filtering here.

The analytic closed-form error-state transition matrix $\Phi(t_{k+1},t_{k})$ is obtained by integrating the following differential equation over the time interval $[t_{k},t_{k+1}]$ 
\begin{equation}\label{Phi_transition_matrix}
\dot{\Phi}(t_{k+1},t_{k})=F\Phi(t_{k+1},t_{k})
\end{equation}
with initial condition $\Phi(t_{k},t_{k})=I_{15}$.

Then, a useful auxiliary function introduce by ~\cite{bloesch2013state} is given
\begin{equation}\label{Gamma}
\Gamma_m(\phi)\triangleq \sum_{n=0}^{\infty}\frac{1}{(n+m)!}(\phi_{\wedge}^n)
\end{equation}
Then the integrals can be easily expressed and computed by the matrix Taylor series 
\begin{equation}\label{Gamma_0}
\Gamma_0(\phi)=I_3+\frac{\sin||\phi||}{||\phi||}\phi_{\wedge}+\frac{1-\cos||\phi||}{||\phi||^2}\phi_{\wedge}^2=R
\end{equation}
\begin{equation}\label{Gamma_1}
\Gamma_1=I_3+\frac{1-\cos||\phi||}{||\phi||^2}\phi_{\wedge}+\frac{||\phi||-\sin||\phi||}{||\phi||^3}\phi_{\wedge}^2=J_l
\end{equation}
It is worth noting that $\Gamma_0(\phi)$ is the exponential mapping of $SO(3)$, while $\Gamma_1(\phi)$ is the left Jacobian of $SO(3)$~\cite{hartley2020contact}.

Since we have  $\Gamma_2(\phi)\phi_{\wedge}+I_3=\Gamma_1(\phi)$ and $\Gamma_2(\phi)$ can be represented as $\Gamma_2(\phi)=\frac{1}{2!}I_3+x\phi_{\wedge}+y\phi_{\wedge}^2$, then $x$ and $y$ can be obtained by determined coefficient method:
\begin{equation}\label{Gamma_2}
\Gamma_2(\phi)=\frac{1}{2}I_3+\frac{||\phi||-\sin||\phi||}{||\phi||^3}\phi_{\wedge}+\frac{||\phi||^2+2\cos||\phi||-2}{2||\phi||^4}\phi_{\wedge}^2
\end{equation}

Similarly, as we have $\Gamma_3(\phi)\phi_{\wedge}+\frac{1}{2}I_3=\Gamma_2(\phi)$ and $\Gamma_3(\phi)$ can be represented as $\Gamma_3(\phi)=\frac{1}{3!}I_3+a\phi_{\wedge}+b\phi_{\wedge}^2$, then $a$ and $b$ can be obtained by determined coefficient method:
\begin{equation}\label{Gamma_3}
\Gamma_3(\phi)=\frac{1}{3!}I_3+\frac{||\phi||^2+2\cos||\phi||-2}{2||\phi||^4}\phi_{\wedge}+\frac{||\phi||^3-6||\phi||+6\sin||\phi||}{6||\phi||^5}\phi_{\wedge}^2
\end{equation}
Therefore
\begin{equation}\label{one_order_integral}
\int_{t_k}^{t_{k+1}}\exp_G((\omega t)_{\wedge})dt=\int_{t_k}^{t_{k+1}}\Gamma_0(\omega t)dt=\left(\sum_{n=0}^{\infty}\frac{1}{(n+1)!}(\omega\Delta t)_{\wedge}^n\right)\Delta t=\Gamma_1(\omega\Delta t)\Delta t
\end{equation}
\begin{equation}\label{two_order_integral}
\begin{aligned}
&\int_{t_k}^{t_{k+1}}\int_{t_k}^{t_k+t_2}\exp_G((\omega t_1)_{\wedge})dt_1dt_2=\int_{t_k}^{t_{k+1}}\Gamma_1(\omega t_2)t_2dt_2\\
=&\left(\sum_{n=0}^{\infty}\frac{1}{(n+2)!}(\omega\Delta t)_{\wedge}^n\right)\Delta t^2
=\Gamma_2(\omega\Delta t)\Delta t^2
\end{aligned}
\end{equation}
\begin{equation}\label{three_order_integral}
\begin{aligned}
&\int_{t_k}^{t_{k+1}}\int_{t_k}^{t_k+t_3}\int_{t_k}^{t_k+t_2}\exp_G((\omega t_1)_{\wedge})dt_1dt_2dt_3=\int_{t_k}^{t_{k+1}}\int_{t_k}^{t_k+t_3}\Gamma_1(\omega t_2)t_2dt_2dt_3\\
=&\int_{t_k}^{t_{k+1}}\Gamma_2(\omega t_3)t_3^2dt_3
=\left(\sum_{n=0}^{\infty}\frac{1}{(n+3)!}(\omega\Delta t)_{\wedge}^n\right)\Delta t^3=\Gamma_3(\omega\Delta t)\Delta t^3
\end{aligned}
\end{equation}

The right invariant error dynamics matrix depends on the estimated attitude $\tilde{C}_b^e$, the estimated $\tilde{v}_{ib}^e$, and the estimated position $\tilde{r}_{ib}^e$. Therefore, it is not independent of the system's trajectory and will change between the times $t_k$ and $t_{k+1}$.
The closed-form of the analytical solution for the right invariant error can be written as
\begin{equation}\label{Phi_transition_matrix_explicit}
\Phi^r(t_{k+1},t_{k})=\begin{bmatrix}
\Phi_{11}^r(t_{k+1},t_{k}) & 0 &0&\Phi_{14}^r(t_{k+1},t_{k}) &0\\
\Phi_{21}^r(t_{k+1},t_{k}) & \Phi_{22}^r(t_{k+1},t_{k})& 0&\Phi_{24}^r(t_{k+1},t_{k})&\Phi_{25}^r(t_{k+1},t_{k})\\
\Phi_{31}^r(t_{k+1},t_{k}) & \Phi_{32}^r(t_{k+1},t_{k})& \Phi_{33}^r(t_{k+1},t_{k})&\Phi_{34}^r(t_{k+1},t_{k})&\Phi_{35}^r(t_{k+1},t_{k})\\
0&0&0&I&0\\
0&0&0&0&I
\end{bmatrix}
\end{equation}

For $\Phi_{11}^r(t_{k+1},t_{k})$
\begin{equation}\label{Phi_11_r}
\begin{aligned}
&\dot{\Phi}_{11}^r(t_{k+1},t_{k})=-\omega_{ie}^e\times\Phi_{11}^r(t_{k+1},t_{k})\\
\Rightarrow& \Phi_{11}^r(t_{k+1},t_{k})=\Phi_{11}^r(t_{k},t_{k})\exp\left(\int_{t_k}^{t_{k+1}}-\omega_{ie}^e\times ds\right)=\exp\left(\int_{t_k}^{t_{k+1}}-\omega_{ie}^e\times ds\right)\\
=&\Gamma_0^T(\omega_{ie}^e\Delta t)={C_e^i}^T=C_i^e
\end{aligned}
\end{equation}

For $\Phi_{14}^r(t_{k+1},t_{k})$
\begin{equation}\label{Phi_14_r}
\begin{aligned}
&\dot{\Phi}_{14}^r(t_{k+1},t_{k})=-\omega_{ie}^e\times\Phi_{14}^r(t_{k+1},t_{k})-\tilde{C}_b^e\\
\Rightarrow& \dot{\Phi}_{14}^r(t_{k+1},t_{k})+\omega_{ie}^e\times\Phi_{14}^r(t_{k+1},t_{k})=-\tilde{C}_b^e\\
\Rightarrow& \Gamma_0(\omega_{ie}^e\Delta t)\dot{\Phi}_{14}^r(t_{k+1},t_{k})+\Gamma_0(\omega_{ie}^e\Delta t)\omega_{ie}^e\times\Phi_{14}^r(t_{k+1},t_{k})=-\Gamma_0(\omega_{ie}^e\Delta t)\tilde{C}_b^e\\
\Rightarrow& \frac{d}{dt}\left(\Gamma_0(\omega_{ie}^e\Delta t) \Phi_{14}^r(t_{k+1},t_{k}) \right)=-C_e^i\tilde{C}_b^e=-\tilde{C}_b^i=-\Gamma_0(\tilde{\omega}_{ib}^b\Delta t)\\
\Rightarrow&\Gamma_0(\omega_{ie}^e\Delta t) \Phi_{14}^r(t_{k+1},t_{k})+\Phi_{14}^r(t_{k},t_{k}) =-\int_{t_k}^{t_{k+1}}\Gamma_0(\tilde{\omega}_{ib}^bs) ds\\
\Rightarrow& \Phi_{14}^r(t_{k+1},t_{k})=-\Gamma_0^T(\omega_{ie}^e\Delta t)\Gamma_1(\tilde{\omega}_{ib}^b\Delta t) \Delta t=-C_i^e\Gamma_1(\tilde{\omega}_{ib}^b\Delta t) \Delta t
\end{aligned}
\end{equation}

For $\Phi_{21}^r(t_{k+1},t_{k})$
\begin{equation}\label{Phi_21_r}
\begin{aligned}
&\dot{\Phi}_{21}^r(t_{k+1},t_{k})=-G_{ib}^e\times\Phi_{11}^r(t_{k+1},t_{k})-\omega_{ie}^e\times\Phi_{21}^r(t_{k+1},t_{k})\\
\Rightarrow& \dot{\Phi}_{21}^r(t_{k+1},t_{k})+\omega_{ie}^e\times\Phi_{21}^r(t_{k+1},t_{k})=-G_{ib}^e\times\Phi_{11}^r(t_{k+1},t_{k})=-G_{ib}^e\times \Gamma_0^T(\omega_{ie}^e\Delta t)\\
\Rightarrow& \Gamma_0(\omega_{ie}^e\Delta t)\dot{\Phi}_{21}^r(t_{k+1},t_{k})+\Gamma_0(\omega_{ie}^e\Delta t)\omega_{ie}^e\times\Phi_{21}^r(t_{k+1},t_{k})=-\Gamma_0(\omega_{ie}^e\Delta t)G_{ib}^e\times\Gamma_0^T(\omega_{ie}^e\Delta t)\\
&=-\left(\Gamma_0(\omega_{ie}^e\Delta t)G_{ib}^e\right)\times\\
\Rightarrow& \frac{d}{dt}\left(\Gamma_0(\omega_{ie}^e\Delta t){\Phi}_{21}^r(t_{k+1},t_{k}) \right)=-\left(\Gamma_0(\omega_{ie}^e\Delta t)G_{ib}^e\right)\times=-(C_e^iG_{ib}^e)\times=-G_{ib}^i\times\\
\Rightarrow& \Gamma_0(\omega_{ie}^e\Delta t){\Phi}_{21}^r(t_{k+1},t_{k})+{\Phi}_{21}^r(t_{k},t_{k})=-\int_{t_k}^{t_{k+1}}\left(\Gamma_0(\omega_{ie}^es)G_{ib}^e\right)\times ds=-\int_{t_k}^{t_{k+1}} (G_{ib}^i\times) ds\\
\Rightarrow&{\Phi}_{21}^r(t_{k+1},t_{k})=-\Gamma_0^T(\omega_{ie}^e\Delta t)\left(\Gamma_1(\omega_{ie}^e\Delta t)G_{ib}^e\right)\times \Delta t\\
&=-\Gamma_0^T(\omega_{ie}^e\Delta t)(G_{ib}^i\times) \Delta t=-C_i^e(G_{ib}^i\times) \Delta t
\end{aligned}
\end{equation}

For $\Phi_{22}^r(t_{k+1},t_{k})$
\begin{equation}\label{Phi_22_r}
\begin{aligned}
&\dot{\Phi}_{22}^r(t_{k+1},t_{k})=-\omega_{ie}^e\times\Phi_{22}^r(t_{k+1},t_{k})\\
\Rightarrow& \Phi_{22}^r(t_{k+1},t_{k})=\Phi_{22}^r(t_{k},t_{k})\exp\left(\int_{t_k}^{t_{k+1}}-\omega_{ie}^e\times ds\right)=\exp\left(\int_{t_k}^{t_{k+1}}-\omega_{ie}^e\times ds\right)\\
&=\Gamma_0^T(\omega_{ie}^e\Delta t)={C_e^i}^T=C_i^e
\end{aligned}
\end{equation}

For $\Phi_{24}^r(t_{k+1},t_{k})$
\begin{equation}\label{Phi_24_r}
\begin{aligned}
&\dot{\Phi}_{24}^r(t_{k+1},t_{k})=-G_{ib}^e\times\Phi_{14}^r(t_{k+1},t_{k})-\omega_{ie}^e\times\Phi_{24}^r(t_{k+1},t_{k})+\tilde{v}_{ib}^e\times \tilde{C}_b^e\\
\Rightarrow&\dot{\Phi}_{24}^r(t_{k+1},t_{k})+\omega_{ie}^e\times\Phi_{24}^r(t_{k+1},t_{k})=-G_{ib}^e\times\left(-\Gamma_0^T(\omega_{ie}^e\Delta t)\Gamma_1(\tilde{\omega}_{ib}^b\Delta t) \Delta t\right)+\tilde{v}_{ib}^e\times \tilde{C}_b^e\\
\Rightarrow&\Gamma_0(\omega_{ie}^e\Delta t)\dot{\Phi}_{24}^r(t_{k+1},t_{k})+\Gamma_0(\omega_{ie}^e\Delta t)\omega_{ie}^e\times\Phi_{24}^r(t_{k+1},t_{k})\\
&=\Gamma_0(\omega_{ie}^e\Delta t)G_{ib}^e\times\left(\Gamma_0^T(\omega_{ie}^e\Delta t)\Gamma_1(\tilde{\omega}_{ib}^b\Delta t) \Delta t\right)+\Gamma_0(\omega_{ie}^e\Delta t)\tilde{v}_{ib}^e\times \tilde{C}_b^e\\
\Rightarrow&\frac{d}{dt}\left( \Gamma_0(\omega_{ie}^e\Delta t){\Phi}_{24}^r(t_{k+1},t_{k})  \right)=\left(\Gamma_0(\omega_{ie}^e\Delta t)G_{ib}^e \right)\times \Gamma_1(\tilde{\omega}_{ib}^b\Delta t) \Delta t+\Gamma_0(\omega_{ie}^e\Delta t)\tilde{v}_{ib}^e\times \tilde{C}_b^e\\
&=G_{ib}^i\times\Gamma_1(\tilde{\omega}_{ib}^b\Delta t) \Delta t+\Gamma_0(\tilde{\omega}_{ib}^b\Delta t)(\tilde{v}_{ib}^b\times)\\
\Rightarrow& \Gamma_0(\omega_{ie}^e\Delta t){\Phi}_{24}^r(t_{k+1},t_{k})+ {\Phi}_{24}^r(t_{k},t_{k}) =(G_{ib}^i\times)\int_{t_k}^{t_{k+1}} \Gamma_1(\tilde{\omega}_{ib}^bs)sds+\int_{t_k}^{t_{k+1}}\Gamma_0(\tilde{\omega}_{ib}^b s)ds(\tilde{v}_{ib}^b\times)\\
\Rightarrow&{\Phi}_{24}^r(t_{k+1},t_{k})=\Gamma_0^T(\omega_{ie}^e\Delta t)(G_{ib}^i\times)\Gamma_2(\tilde{\omega}_{ib}^b\Delta t) \Delta t^2+\Gamma_0^T(\omega_{ie}^e\Delta t)\Gamma_1(\tilde{\omega}_{ib}^b\Delta t)\Delta t(\tilde{v}_{ib}^b\times)\\
&=C_i^e(G_{ib}^i\times)\Gamma_2(\tilde{\omega}_{ib}^b\Delta t) \Delta t^2+C_i^e\Gamma_1(\tilde{\omega}_{ib}^b\Delta t)\Delta t(\tilde{v}_{ib}^b\times)
\end{aligned}
\end{equation}

For $\Phi_{25}^r(t_{k+1},t_{k})$
\begin{equation}\label{Phi_25_r}
\begin{aligned}
&\dot{\Phi}_{25}^r(t_{k+1},t_{k})=-\omega_{ie}^e\times\Phi_{25}^r(t_{k+1},t_{k})+\tilde{C}_b^e\\
\Rightarrow&\dot{\Phi}_{25}^r(t_{k+1},t_{k})+\omega_{ie}^e\times\Phi_{25}^r(t_{k+1},t_{k})= \tilde{C}_b^e\\
\Rightarrow&\Gamma_0(\omega_{ie}^e\Delta t)\dot{\Phi}_{25}^r(t_{k+1},t_{k})+\Gamma_0(\omega_{ie}^e\Delta t)\omega_{ie}^e\times\Phi_{25}^r(t_{k+1},t_{k})=\Gamma_0(\omega_{ie}^e\Delta t) \tilde{C}_b^e=\Gamma_0(\omega_{ib}^b\Delta t)\\
\Rightarrow&\frac{d}{dt}\left( \Gamma_0(\omega_{ie}^e\Delta t){\Phi}_{24}^r(t_{k+1},t_{k})  \right)=\Gamma_0(\omega_{ib}^b\Delta t)\\
\Rightarrow&\Gamma_0(\omega_{ie}^e\Delta t){\Phi}_{25}^r(t_{k+1},t_{k})+{\Phi}_{25}^r(t_{k},t_{k})=\int_{t_k}^{t_{k+1}}\Gamma_0(\omega_{ib}^bs) ds\\
\Rightarrow&{\Phi}_{25}^r(t_{k+1},t_{k})=\Gamma_0^T(\omega_{ie}^e\Delta t)\Gamma_1(\tilde{\omega}_{ib}^b\Delta t)\Delta t=C_i^e \Gamma_1(\tilde{\omega}_{ib}^b\Delta t)\Delta t
\end{aligned}
\end{equation}

For $\Phi_{31}^r(t_{k+1},t_{k})$
\begin{equation}\label{Phi31_r}
\begin{aligned}
&\dot{\Phi}_{31}^r(t_{k+1},t_{k})=\Phi_{21}^r(t_{k+1},t_{k})-\omega_{ie}^e\times\Phi_{31}^r(t_{k+1},t_{k})\\
\Rightarrow& \dot{\Phi}_{31}^r(t_{k+1},t_{k})+\omega_{ie}^e\times\Phi_{31}^r(t_{k+1},t_{k})=\Phi_{21}^r(t_{k+1},t_{k})=-C_i^e(G_{ib}^i\times) \Delta t\\
\Rightarrow& C_e^i\dot{\Phi}_{31}^r(t_{k+1},t_{k})+C_e^i\omega_{ie}^e\times\Phi_{31}^r(t_{k+1},t_{k})=-(G_{ib}^i\times) \Delta t\\
\Rightarrow& \frac{d}{dt}\left(C_e^i{\Phi}_{31}^r(t_{k+1},t_{k}) \right)=-(G_{ib}^i\times) \Delta t\\
\Rightarrow& C_e^i{\Phi}_{31}^r(t_{k+1},t_{k})+{\Phi}_{31}^r(t_{k},t_{k})=-(G_{ib}^i\times)\int_{t_k}^{t_{k+1}} Is ds=-(G_{ib}^i\times)\frac{1}{2}\Delta t^2\\
\Rightarrow& {\Phi}_{31}^r(t_{k+1},t_{k})=-C_i^e\frac{1}{2}(G_{ib}^i\times)\Delta t^2
\end{aligned}
\end{equation}

For $\Phi_{32}^r(t_{k+1},t_{k})$
\begin{equation}\label{Phi32_r}
\begin{aligned}
&\dot{\Phi}_{32}^r(t_{k+1},t_{k})=\Phi_{22}^r(t_{k+1},t_{k})-\omega_{ie}^e\times\Phi_{32}^r(t_{k+1},t_{k})\\
\Rightarrow& \dot{\Phi}_{32}^r(t_{k+1},t_{k})+\omega_{ie}^e\times\Phi_{32}^r(t_{k+1},t_{k})=\Phi_{22}^r(t_{k+1},t_{k})=C_i^e\\
\Rightarrow& C_e^i\dot{\Phi}_{31}^r(t_{k+1},t_{k})+C_e^i\omega_{ie}^e\times\Phi_{31}^r(t_{k+1},t_{k})=I\\
\Rightarrow& \frac{d}{dt}\left(C_e^i{\Phi}_{31}^r(t_{k+1},t_{k}) \right)=I \\
\Rightarrow& {\Phi}_{31}^r(t_{k+1},t_{k})=C_i^e\Delta t
\end{aligned}
\end{equation}

For $\Phi_{33}^r(t_{k+1},t_{k})$
\begin{equation}\label{Phi33_r}
\begin{aligned}
&\dot{\Phi}_{33}^r(t_{k+1},t_{k})=-\omega_{ie}^e\times\Phi_{33}^r(t_{k+1},t_{k})\\
\Rightarrow& \Phi_{33}^r(t_{k+1},t_{k})=\Phi_{33}^r(t_{k},t_{k})\exp\left(\int_{t_k}^{t_{k+1}}-\omega_{ie}^e\times ds\right)=\exp\left(\int_{t_k}^{t_{k+1}}-\omega_{ie}^e\times ds\right)\\
&=\Gamma_0^T(\omega_{ie}^e\Delta t)={C_e^i}^T=C_i^e
\end{aligned}
\end{equation}

For $\Phi_{34}^r(t_{k+1},t_{k})$
\begin{equation}\label{Phi_34_r}
\begin{aligned}
&\dot{\Phi}_{34}^r(t_{k+1},t_{k})=\Phi_{24}^r(t_{k+1},t_{k})-\omega_{ie}^e\times\Phi_{34}^r(t_{k+1},t_{k})+\tilde{p}_{ib}^e\times \tilde{C}_b^e\\
\Rightarrow&\dot{\Phi}_{34}^r(t_{k+1},t_{k})+\omega_{ie}^e\times\Phi_{34}^r(t_{k+1},t_{k})=\Phi_{24}^r(t_{k+1},t_{k})+\tilde{p}_{ib}^e\times \tilde{C}_b^e\\
\Rightarrow&\Gamma_0(\omega_{ie}^e\Delta t)\dot{\Phi}_{34}^r(t_{k+1},t_{k})+\Gamma_0(\omega_{ie}^e\Delta t)\omega_{ie}^e\times\Phi_{34}^r(t_{k+1},t_{k})\\
&=\Gamma_0(\omega_{ie}^e\Delta t)\Phi_{24}^r(t_{k+1},t_{k})+\Gamma_0(\omega_{ie}^e\Delta t)\tilde{p}_{ib}^e\times \tilde{C}_b^e\\
\Rightarrow&\frac{d}{dt}\left( \Gamma_0(\omega_{ie}^e\Delta t){\Phi}_{24}^r(t_{k+1},t_{k})  \right)=(G_{ib}^i\times)\Gamma_2(\tilde{\omega}_{ib}^b\Delta t) \Delta t^2+\Gamma_1(\tilde{\omega}_{ib}^b\Delta t)\Delta t(\tilde{v}_{ib}^b\times)+\Gamma_0(\tilde{\omega}_{ib}^b\Delta t)(\tilde{p}_{ib}^b\times)\\
\Rightarrow&{\Phi}_{34}^r(t_{k+1},t_{k})=\Gamma_0^T(\omega_{ie}^e\Delta t)(G_{ib}^i\times)\Gamma_3(\tilde{\omega}_{ib}^b\Delta t) \Delta t^3
+\Gamma_0^T(\omega_{ie}^e\Delta t)\Gamma_2(\tilde{\omega}_{ib}^b\Delta t)\Delta t^2(\tilde{v}_{ib}^b\times)\\
&+\Gamma_0^T(\omega_{ie}^e\Delta t)\Gamma_1(\tilde{\omega}_{ib}^b\Delta t)\Delta t(\tilde{p}_{ib}^b\times)\\
&=C_i^e(G_{ib}^i\times)\Gamma_3(\tilde{\omega}_{ib}^b\Delta t) \Delta t^3+C_i^e\Gamma_2(\tilde{\omega}_{ib}^b\Delta t)\Delta t^2(\tilde{v}_{ib}^b\times)+C_i^e\Gamma_1(\tilde{\omega}_{ib}^b\Delta t)\Delta t(\tilde{p}_{ib}^b\times)
\end{aligned}
\end{equation}

For $\Phi_{35}^r(t_{k+1},t_{k})$
\begin{equation}\label{Phi_35_r}
\begin{aligned}
&\dot{\Phi}_{35}^r(t_{k+1},t_{k})=\Phi_{25}^r(t_{k+1},t_{k})-\omega_{ie}^e\times\Phi_{35}^r(t_{k+1},t_{k})\\
\Rightarrow&\dot{\Phi}_{35}^r(t_{k+1},t_{k})+\omega_{ie}^e\times\Phi_{35}^r(t_{k+1},t_{k})=\Phi_{25}^r(t_{k+1},t_{k})\\
\Rightarrow&\Gamma_0(\omega_{ie}^e\Delta t)\dot{\Phi}_{35}^r(t_{k+1},t_{k})+\Gamma_0(\omega_{ie}^e\Delta t)\omega_{ie}^e\times\Phi_{35}^r(t_{k+1},t_{k})\\
&=\Gamma_0(\omega_{ie}^e\Delta t)C_i^e \Gamma_1(\tilde{\omega}_{ib}^b\Delta t)\Delta t=\Gamma_1(\tilde{\omega}_{ib}^b\Delta t)\Delta t\\
\Rightarrow&\frac{d}{dt}\left( \Gamma_0(\omega_{ie}^e\Delta t){\Phi}_{35}^r(t_{k+1},t_{k})  \right)=\Gamma_1(\tilde{\omega}_{ib}^b\Delta t)\Delta t\\
\Rightarrow& \Gamma_0(\omega_{ie}^e\Delta t){\Phi}_{35}^r(t_{k+1},t_{k}) +{\Phi}_{35}^r(t_{k},t_{k})=\int_{t_k}^{t_{k+1}}\Gamma_1(\tilde{\omega}_{ib}^bs)sds\\
\Rightarrow&{\Phi}_{35}^r(t_{k+1},t_{k})=\Gamma_0^T(\omega_{ie}^e\Delta t)\Gamma_2(\tilde{\omega}_{ib}^b\Delta t)\Delta t^2=C_i^e\Gamma_2(\tilde{\omega}_{ib}^b\Delta t)\Delta t^2
\end{aligned}
\end{equation}

On the other hand, the left invariant error dynamics matrix in equation (\ref{state_x_estimated}) only depends on the acceleration and gyroscope readings and the estimated bias terms and can be assumed to be constant between times $t_k$ and $t_{k+1}$.
Therefore, the closed-form of the analytical solution for the left invariant error can be written as
\begin{equation}\label{Phi_transition_matrix_explicit_left}
\Phi^l(t_{k+1},t_{k})=\begin{bmatrix}
\Phi_{11}^l(t_{k+1},t_{k}) & 0 &0&\Phi_{14}^l(t_{k+1},t_{k}) &0\\
\Phi_{21}^l(t_{k+1},t_{k}) & \Phi_{22}^l(t_{k+1},t_{k})& 0&\Phi_{24}^l(t_{k+1},t_{k})&\Phi_{25}^l(t_{k+1},t_{k})\\
\Phi_{31}^l(t_{k+1},t_{k}) & \Phi_{32}^l(t_{k+1},t_{k})& \Phi_{33}^l(t_{k+1},t_{k})&\Phi_{34}^l(t_{k+1},t_{k})&\Phi_{35}^l(t_{k+1},t_{k})\\
0&0&0&I&0\\
0&0&0&0&I
\end{bmatrix}
\end{equation}
where the individual terms can be analytically calculated in details as following.

For $\Phi_{11}^l(t_{k+1},t_{k})$
\begin{equation}\label{Phi_11_l}
\begin{aligned}
&\dot{\Phi}_{11}^l(t_{k+1},t_{k})=-\tilde{\omega}_{ib}^b\times\Phi_{11}^l(t_{k+1},t_{k})\\
\Rightarrow& \Phi_{11}^l(t_{k+1},t_{k})=\Phi_{11}^l(t_{k},t_{k})\exp\left(\int_{t_k}^{t_{k+1}}-\tilde{\omega}_{ib}^b\times ds\right)=\exp\left(\int_{t_k}^{t_{k+1}}-\tilde{\omega}_{ib}^b\times ds\right)\\
&=\Gamma_0^T(\tilde{\omega}_{ib}^b\Delta t)
\end{aligned}
\end{equation}

For $\Phi_{14}^l(t_{k+1},t_{k})$
\begin{equation}\label{Phi_14_l}
\begin{aligned}
&\dot{\Phi}_{14}^l(t_{k+1},t_{k})=-\tilde{\omega}_{ib}^b\times\Phi_{14}^l(t_{k+1},t_{k})-I\\
\Rightarrow& \dot{\Phi}_{14}^l(t_{k+1},t_{k})+\tilde{\omega}_{ib}^b\times\Phi_{14}^l(t_{k+1},t_{k})=-I\\
\Rightarrow& \Gamma_0(\tilde{\omega}_{ib}^b\Delta t)\dot{\Phi}_{14}^l(t_{k+1},t_{k})+\Gamma_0(\tilde{\omega}_{ib}^b\Delta t)\tilde{\omega}_{ib}^b\times\Phi_{14}^l(t_{k+1},t_{k})=-\Gamma_0(\tilde{\omega}_{ib}^b\Delta t)\\
\Rightarrow& \frac{d}{dt}\left(\Gamma_0(\tilde{\omega}_{ib}^b\Delta t)\Phi_{14}^l(t_{k+1},t_{k})\right)=-\Gamma_0(\tilde{\omega}_{ib}^b\Delta t)\\
\Rightarrow&\Gamma_0(\tilde{\omega}_{ib}^b\Delta t)\Phi_{14}^l(t_{k+1},t_{k})+\Phi_{14}^l(t_{k},t_{k})=-\int_{t_k}^{t_{k+1}} \Gamma_0(\tilde{\omega}_{ib}^bs) ds=-\Gamma_1(\tilde{\omega}_{ib}^b\Delta t)\Delta t\\
\Rightarrow&\Phi_{14}^l(t_{k+1},t_{k})=-\Gamma_0^T(\tilde{\omega}_{ib}^b\Delta t)\Gamma_1(\tilde{\omega}_{ib}^b\Delta t)\Delta t
\end{aligned}
\end{equation}

For $\Phi_{21}^l(t_{k+1},t_{k})$
\begin{equation}\label{Phi_21_l}
\begin{aligned}
&\dot{\Phi}_{21}^l(t_{k+1},t_{k})=-\tilde{f}^b\times\Phi_{11}^l(t_{k+1},t_{k})-\tilde{\omega}_{ib}^b\times\Phi_{21}^l(t_{k+1},t_{k})\\
\Rightarrow& \dot{\Phi}_{21}^l(t_{k+1},t_{k})+\tilde{\omega}_{ib}^b\times\Phi_{21}^l(t_{k+1},t_{k})=-\tilde{f}^b\times\Phi_{11}^l(t_{k+1},t_{k})=-(\tilde{f}^b\times) \Gamma_0^T(\tilde{\omega}_{ib}^b\Delta t)\\
\Rightarrow& \Gamma_0(\tilde{\omega}_{ib}^b\Delta t)\dot{\Phi}_{21}^l(t_{k+1},t_{k})+\Gamma_0(\tilde{\omega}_{ib}^b\Delta t)\tilde{\omega}_{ib}^b\times\Phi_{21}^l(t_{k+1},t_{k})\\
&=-\Gamma_0(\tilde{\omega}_{ib}^b\Delta t)(\tilde{f}^b\times) \Gamma_0^T(\tilde{\omega}_{ib}^b\Delta t)=-(\Gamma_0(\tilde{\omega}_{ib}^b\Delta t)\tilde{f}^b)\times\\
\Rightarrow&  \frac{d}{dt}\left(\Gamma_0(\tilde{\omega}_{ib}^b\Delta t)\Phi_{21}^l(t_{k+1},t_{k})\right)=-(\Gamma_0(\tilde{\omega}_{ib}^b\Delta t)\tilde{f}^b)\times\\
\Rightarrow& \Gamma_0(\tilde{\omega}_{ib}^b\Delta t)\Phi_{21}^l(t_{k+1},t_{k})+\Phi_{21}^l(t_{k},t_{k})=-\int_{t_k}^{t_{k+1}}(\Gamma_0(\tilde{\omega}_{ib}^b s)\tilde{f}^b)\times ds\\
\Rightarrow&\Phi_{21}^l(t_{k+1},t_{k})= -\Gamma_0^T(\tilde{\omega}_{ib}^b\Delta t)(\Gamma_1(\tilde{\omega}_{ib}^b\Delta t)\tilde{f}^b)\times \Delta t
\end{aligned}
\end{equation}

For $\Phi_{22}^l(t_{k+1},t_{k})$
\begin{equation}\label{Phi_22_l}
\begin{aligned}
&\dot{\Phi}_{22}^l(t_{k+1},t_{k})=-\tilde{\omega}_{ib}^b\times\Phi_{22}^l(t_{k+1},t_{k})\\
\Rightarrow& \Phi_{22}^l(t_{k+1},t_{k})=\Phi_{22}^l(t_{k},t_{k})\exp\left(\int_{t_k}^{t_{k+1}}-\tilde{\omega}_{ib}^b\times ds\right)=\exp\left(\int_{t_k}^{t_{k+1}}-\tilde{\omega}_{ib}^b\times ds\right)\\
&=\Gamma_0^T(\tilde{\omega}_{ib}^b\Delta t)
\end{aligned}
\end{equation}

For $\Phi_{24}^l(t_{k+1},t_{k})$
\begin{equation}\label{Phi_24_l}
\begin{aligned}
&\dot{\Phi}_{24}^l(t_{k+1},t_{k})=-\tilde{f}^b\times\Phi_{14}^l(t_{k+1},t_{k})-\tilde{\omega}_{ib}^b\times\Phi_{24}^l(t_{k+1},t_{k})\\
\Rightarrow& \dot{\Phi}_{24}^l(t_{k+1},t_{k})+\tilde{\omega}_{ib}^b\times\Phi_{24}^l(t_{k+1},t_{k})=-\tilde{f}^b\times\Phi_{14}^l(t_{k+1},t_{k})\\
&=-(\tilde{f}^b\times) \left(-\Gamma_0^T(\tilde{\omega}_{ib}^b\Delta t)\Gamma_1(\tilde{\omega}_{ib}^b\Delta t)\Delta t \right)\\
\Rightarrow& \Gamma_0(\tilde{\omega}_{ib}^b\Delta t)\dot{\Phi}_{24}^l(t_{k+1},t_{k})+\Gamma_0(\tilde{\omega}_{ib}^b\Delta t)\tilde{\omega}_{ib}^b\times\Phi_{24}^l(t_{k+1},t_{k})\\
&=-\Gamma_0(\tilde{\omega}_{ib}^b\Delta t)(\tilde{f}^b\times) \left(-\Gamma_0^T(\tilde{\omega}_{ib}^b\Delta t)\Gamma_1(\tilde{\omega}_{ib}^b\Delta t)\Delta t \right)\\
&=(\Gamma_0(\tilde{\omega}_{ib}^b\Delta t)\tilde{f}^b)\times \Gamma_1(\tilde{\omega}_{ib}^b\Delta t)\Delta t\\
\Rightarrow&  \frac{d}{dt}\left(\Gamma_0(\tilde{\omega}_{ib}^b\Delta t)\Phi_{24}^l(t_{k+1},t_{k})\right)=(\Gamma_0(\tilde{\omega}_{ib}^b\Delta t)\tilde{f}^b)\times \Gamma_1(\tilde{\omega}_{ib}^b\Delta t)\Delta t\\
\Rightarrow& \Gamma_0(\tilde{\omega}_{ib}^b\Delta t)\Phi_{24}^l(t_{k+1},t_{k})+\Phi_{24}^l(t_{k},t_{k})=\int_{t_k}^{t_{k+1}}(\Gamma_0(\tilde{\omega}_{ib}^bs )\tilde{f}^b)\times \Gamma_1(\tilde{\omega}_{ib}^bs)s ds\\
\Rightarrow&\Phi_{24}^l(t_{k+1},t_{k})= \Gamma_0^T(\tilde{\omega}_{ib}^b\Delta t)\Psi_1
\end{aligned}
\end{equation}
where $\Psi_1\triangleq \int_{t_k}^{t_{k+1}}(\Gamma_0(\tilde{\omega}_{ib}^bs)\tilde{f}^b)\times \Gamma_1(\tilde{\omega}_{ib}^bs)s ds$ is defined and calculated in~\cite{hartley2020contact}.

For $\Phi_{25}^l(t_{k+1},t_{k})$
\begin{equation}\label{Phi_25_l}
\begin{aligned}
&\dot{\Phi}_{25}^l(t_{k+1},t_{k})=-\tilde{\omega}_{ib}^b\times\Phi_{25}^l(t_{k+1},t_{k})-I\\
\Rightarrow& \dot{\Phi}_{25}^l(t_{k+1},t_{k})+\tilde{\omega}_{ib}^b\times\Phi_{25}^l(t_{k+1},t_{k})=-I\\
\Rightarrow& \Gamma_0(\tilde{\omega}_{ib}^b\Delta t)\dot{\Phi}_{25}^l(t_{k+1},t_{k})+\Gamma_0(\tilde{\omega}_{ib}^b\Delta t)\tilde{\omega}_{ib}^b\times\Phi_{25}^l(t_{k+1},t_{k})=-\Gamma_0(\tilde{\omega}_{ib}^b\Delta t)\\
\Rightarrow& \frac{d}{dt}\left(\Gamma_0(\tilde{\omega}_{ib}^b\Delta t)\Phi_{25}^l(t_{k+1},t_{k})\right)=-\Gamma_0(\tilde{\omega}_{ib}^b\Delta t)\\
\Rightarrow&\Gamma_0(\tilde{\omega}_{ib}^b\Delta t)\Phi_{25}^l(t_{k+1},t_{k})+\Phi_{25}^l(t_{k},t_{k})=-\int_{t_k}^{t_{k+1}} \Gamma_0(\tilde{\omega}_{ib}^bs) ds=-\Gamma_1(\tilde{\omega}_{ib}^b\Delta t)\Delta t\\
\Rightarrow&\Phi_{25}^l(t_{k+1},t_{k})=-\Gamma_0^T(\tilde{\omega}_{ib}^b\Delta t)\Gamma_1(\tilde{\omega}_{ib}^b\Delta t)\Delta t
\end{aligned}
\end{equation}

For $\Phi_{31}^l(t_{k+1},t_{k})$
\begin{equation}\label{phi_31_l}
\begin{aligned}
&\dot{\Phi}_{31}^l(t_{k+1},t_{k})=\Phi_{21}^l(t_{k+1},t_{k})-\tilde{\omega}_{ib}^b\times\Phi_{31}(t_{k+1},t_{k})\\
\Rightarrow & \dot{\Phi}_{31}^l(t_{k+1},t_{k})+\tilde{\omega}_{ib}^b\times\Phi_{31}(t_{k+1},t_{k})=\Phi_{21}^l(t_{k+1},t_{k})=-\Gamma_0^T(\tilde{\omega}_{ib}^b\Delta t)(\Gamma_1(\tilde{\omega}_{ib}^b\Delta t)\tilde{f}^b)\times \Delta t\\
\Rightarrow & \Gamma_0(\tilde{\omega}_{ib}^b\Delta t)\dot{\Phi}_{31}^l(t_{k+1},t_{k})+\Gamma_0(\tilde{\omega}_{ib}^b\Delta t)\tilde{\omega}_{ib}^b\times\Phi_{31}(t_{k+1},t_{k})=-(\Gamma_1(\tilde{\omega}_{ib}^b\Delta t)\tilde{f}^b)\times \Delta t\\
\Rightarrow & \frac{d}{dt}\left(  \Gamma_0(\tilde{\omega}_{ib}^b\Delta t){\Phi}_{31}^l(t_{k+1},t_{k})  \right)=-(\Gamma_1(\tilde{\omega}_{ib}^b\Delta t)\tilde{f}^b)\times \Delta t\\
\Rightarrow & \Gamma_0(\tilde{\omega}_{ib}^b\Delta t){\Phi}_{31}^l(t_{k+1},t_{k}) +{\Phi}_{31}^l(t_{k},t_{k})=-\int_{t_k}^{t_{k+1}}(\Gamma_1(\tilde{\omega}_{ib}^bs)\tilde{f}^b)\times s ds\\
\Rightarrow &{\Phi}_{31}^l(t_{k+1},t_{k})=- \Gamma_0^T(\tilde{\omega}_{ib}^b\Delta t)(\Gamma_2(\tilde{\omega}_{ib}^b\Delta t)\tilde{f}^b)\times \Delta t^2
\end{aligned}
\end{equation}

For $\Phi_{32}^l(t_{k+1},t_{k})$
\begin{equation}\label{phi_32_l}
\begin{aligned}
&\dot{\Phi}_{32}^l(t_{k+1},t_{k})=\Phi_{22}^l(t_{k+1},t_{k})-\tilde{\omega}_{ib}^b\times\Phi_{32}(t_{k+1},t_{k})\\
\Rightarrow & \dot{\Phi}_{32}^l(t_{k+1},t_{k})+\tilde{\omega}_{ib}^b\times\Phi_{32}(t_{k+1},t_{k})=\Phi_{22}^l(t_{k+1},t_{k})=\Gamma_0^T(\tilde{\omega}_{ib}^b\Delta t)\\
\Rightarrow & \Gamma_0(\tilde{\omega}_{ib}^b\Delta t)\dot{\Phi}_{32}^l(t_{k+1},t_{k})+\Gamma_0(\tilde{\omega}_{ib}^b\Delta t)\tilde{\omega}_{ib}^b\times\Phi_{32}(t_{k+1},t_{k})=I\\
\Rightarrow & \frac{d}{dt}\left(  \Gamma_0(\tilde{\omega}_{ib}^b\Delta t){\Phi}_{32}^l(t_{k+1},t_{k})  \right)=I\\
\Rightarrow & \Gamma_0(\tilde{\omega}_{ib}^b\Delta t){\Phi}_{32}^l(t_{k+1},t_{k}) +{\Phi}_{32}^l(t_{k},t_{k})=\int_{t_k}^{t_{k+1}}I ds\\
\Rightarrow &{\Phi}_{32}^l(t_{k+1},t_{k})=\Gamma_0^T(\tilde{\omega}_{ib}^b\Delta t) \Delta t
\end{aligned}
\end{equation}

For $\Phi_{33}^l(t_{k+1},t_{k})$
\begin{equation}\label{Phi_33_l}
\begin{aligned}
&\dot{\Phi}_{33}^l(t_{k+1},t_{k})=-\tilde{\omega}_{ib}^b\times\Phi_{33}^l(t_{k+1},t_{k})\\
\Rightarrow& \Phi_{33}^l(t_{k+1},t_{k})=\Phi_{33}^l(t_{k},t_{k})\exp\left(\int_{t_k}^{t_{k+1}}-\tilde{\omega}_{ib}^b\times ds\right)=\exp\left(\int_{t_k}^{t_{k+1}}-\tilde{\omega}_{ib}^b\times ds\right)\\
&=\Gamma_0^T(\tilde{\omega}_{ib}^b\Delta t)
\end{aligned}
\end{equation}

For $\Phi_{34}^l(t_{k+1},t_{k})$
\begin{equation}\label{phi_34_l}
\begin{aligned}
&\dot{\Phi}_{34}^l(t_{k+1},t_{k})=\Phi_{24}^l(t_{k+1},t_{k})-\tilde{\omega}_{ib}^b\times\Phi_{34}(t_{k+1},t_{k})\\
\Rightarrow & \dot{\Phi}_{34}^l(t_{k+1},t_{k})+\tilde{\omega}_{ib}^b\times\Phi_{34}(t_{k+1},t_{k})=\Phi_{24}^l(t_{k+1},t_{k})=\Gamma_0^T(\tilde{\omega}_{ib}^b\Delta t)\Psi_1\\
\Rightarrow & \Gamma_0(\tilde{\omega}_{ib}^b\Delta t)\dot{\Phi}_{34}^l(t_{k+1},t_{k})+\Gamma_0(\tilde{\omega}_{ib}^b\Delta t)\tilde{\omega}_{ib}^b\times\Phi_{34}(t_{k+1},t_{k})=\Psi_1\\
\Rightarrow & \frac{d}{dt}\left(  \Gamma_0(\tilde{\omega}_{ib}^b\Delta t){\Phi}_{34}^l(t_{k+1},t_{k})  \right)=\Psi_1\\
\Rightarrow & \Gamma_0(\tilde{\omega}_{ib}^b\Delta t){\Phi}_{34}^l(t_{k+1},t_{k}) +{\Phi}_{34}^l(t_{k},t_{k})=\int_{t_k}^{t_{k+1}}\Psi_1 ds\\
\Rightarrow &{\Phi}_{34}^l(t_{k+1},t_{k})= \Gamma_0^T(\tilde{\omega}_{ib}^b\Delta t)\Psi_2
\end{aligned}
\end{equation}
where $\Psi_2\triangleq \int_{t_k}^{t_{k+1}}\Psi_1 ds$ is defined and calculated in~\cite{hartley2020contact}.

For $\Phi_{35}^l(t_{k+1},t_{k})$
\begin{equation}\label{phi_35_l}
\begin{aligned}
&\dot{\Phi}_{35}^l(t_{k+1},t_{k})=\Phi_{25}^l(t_{k+1},t_{k})-\tilde{\omega}_{ib}^b\times\Phi_{35}(t_{k+1},t_{k})\\
\Rightarrow & \dot{\Phi}_{35}^l(t_{k+1},t_{k})+\tilde{\omega}_{ib}^b\times\Phi_{35}(t_{k+1},t_{k})=\Phi_{25}^l(t_{k+1},t_{k})=-\Gamma_0^T(\tilde{\omega}_{ib}^b\Delta t)\Gamma_1(\tilde{\omega}_{ib}^b\Delta t)\Delta t\\
\Rightarrow & \Gamma_0(\tilde{\omega}_{ib}^b\Delta t)\dot{\Phi}_{35}^l(t_{k+1},t_{k})+\Gamma_0(\tilde{\omega}_{ib}^b\Delta t)\tilde{\omega}_{ib}^b\times\Phi_{35}(t_{k+1},t_{k})=-\Gamma_1(\tilde{\omega}_{ib}^b\Delta t)\Delta t\\
\Rightarrow & \frac{d}{dt}\left(  \Gamma_0(\tilde{\omega}_{ib}^b\Delta t){\Phi}_{35}^l(t_{k+1},t_{k})  \right)=-\Gamma_1(\tilde{\omega}_{ib}^b\Delta t)\Delta t\\
\Rightarrow & \Gamma_0(\tilde{\omega}_{ib}^b\Delta t){\Phi}_{35}^l(t_{k+1},t_{k}) +{\Phi}_{35}^l(t_{k},t_{k})=-\int_{t_k}^{t_{k+1}}\Gamma_1(\tilde{\omega}_{ib}^bs)s ds\\
\Rightarrow &{\Phi}_{35}^l(t_{k+1},t_{k})=-\Gamma_0^T(\tilde{\omega}_{ib}^b\Delta t) \Gamma_2(\tilde{\omega}_{ib}^b\Delta t)\Delta t^2
\end{aligned}
\end{equation}

The observability matrix is computed as~\cite{bloesch2013state}
\begin{equation}\label{Observability_matrix}
M=\begin{bmatrix}
H_k\\
H_{k+1}\Phi_{k+1,k}\\
\vdots\\
H_l\Psi_{l,k}\\
\vdots\\
H_{k+m}\Phi_{k+m,k}
\end{bmatrix}
\end{equation}
where $\Phi_{l,k}$ is the  error state transition matrix from time $t_k$ to $t_l$, and $H_l$ is the measurement Jacobian matrix at time $t_l$.

Based on the left measurement Jacobians, the $l-th$ block row of the discrete-time observability matrix $M$ for left invariant error can be obtained 
\begin{equation}\label{M_l_th_block}
\begin{aligned}
&M_l^l=H_l^l\Phi_{l,k}^l=C_b^e\begin{bmatrix}
-(l^b\times)\Phi_{11}^l+\Phi_{31}^l & \Phi_{32}^l & \Phi_{33}^l& -(l^b\times)\Phi_{14}^l+\Phi_{34}^l& \Phi_{35}^l]
\end{bmatrix}\\
=&\Pi\begin{bmatrix}
-(\Gamma_0(\tilde{\omega}_{ib}^b\Delta t)l^b)\times-(\Gamma_2(\tilde{\omega}_{ib}^b\Delta t)\tilde{f}^b)\times \Delta t^2\quad  I\Delta t \quad I\\
 (\Gamma_0(\tilde{\omega}_{ib}^b\Delta t)l^b)\times\Gamma_1(\tilde{\omega}_{ib}^b\Delta t)\Delta t+\Psi_2\quad -\Gamma_2(\tilde{\omega}_{ib}^b\Delta t)\Delta t^2
\end{bmatrix}
\end{aligned}
\end{equation}
where $\Pi=C_b^e\Gamma_0^T(\tilde{\omega}_{ib}^b\Delta t)$.

Based on the right measurement Jacobians, the $l-th$ block row of the discrete-time observability matrix $M$ for right invariant error can be obtained 
\begin{equation}\label{M_r_th_block}
\begin{aligned}
&M_l^r=H_l^r\Phi_{l,k}^r=\begin{bmatrix}
-(\tilde{r}_{ib}^e+C_b^el^b) \times\Phi_{11}^r+\Phi_{31}^r & \Phi_{32}^r & \Phi_{33}^r& -(\tilde{r}_{ib}^e+C_b^el^b) \times\Phi_{14}^r+\Phi_{34}^r& \Phi_{35}^r]
\end{bmatrix}\\
=&C_i^e\begin{bmatrix}
-(C_e^i\tilde{r}_{ib}^e+C_e^iC_b^el^b)\times-\frac{1}{2}G_{ib}^i\times \Delta t^2\quad I\Delta t \quad I\\ ((C_e^i\tilde{r}_{ib}^e+C_e^iC_b^el^b)\times\Gamma_1(\tilde{\omega}_{ib}^b\Delta t)\Delta t+\Theta\quad \Gamma_2(\tilde{\omega}_{ib}^b\Delta t)\Delta t^2
\end{bmatrix}\\
=&C_i^e\begin{bmatrix}
-(\tilde{r}_{ib}^i+C_b^il^b)\times-\frac{1}{2}G_{ib}^i\times  \Delta t^2 \quad I\Delta t \quad I\\ ((\tilde{r}_{ib}^i+C_b^il^b)\times\Gamma_1(\tilde{\omega}_{ib}^b\Delta t)\Delta t+\Theta\quad \Gamma_2(\tilde{\omega}_{ib}^b\Delta t)\Delta t^2
\end{bmatrix}
\end{aligned}
\end{equation}
where $\Theta=\left((G_{ib}^i\times)\Gamma_3(\tilde{\omega}_{ib}^b\Delta t) \Delta t^3+\Gamma_2(\tilde{\omega}_{ib}^b\Delta t)\Delta t^2(\tilde{v}_{ib}^b\times) +\Gamma_1(\tilde{\omega}_{ib}^b\Delta t)\Delta t(\tilde{p}_{ib}^b\times)\right)$.

We can find that $M$ Is rank deficient by 1 (i.e., the dimension of its nullspace is 1). This is due to the gravity vector (yaw) is unobservable. However, since the absolute position measurements can be provided by the GNSS, the absolute position is observable.
\section{Conclusion}
In this paper, the left equivariant property of the inertial-integrated kinematics system is exploited for the design of equivariant filtering framework. 
It is can be termed as a specialization of the equivariant filtering for second order kinematic systems on $\mathrm{T}\mathbf{SE}_2(3)$. Meanwhile, the analytical state transition matrices are given in details for left invariant error and right invariant error.
\vspace{2ex}

\noindent
{\bf\normalsize Acknowledgement}\newline
{This research was supported by a grant from the National Key Research and Development Program of China (2018YFB1305001).} \vspace{2ex}

\bibliographystyle{IEEEtran}
\bibliography{sample.bib}

\end{document}